\documentclass[letterpaper, 10 pt, journal, twoside]{IEEEtran}  % Comment this line out if you need a4paper

\usepackage{caption}
\usepackage{graphicx}
\usepackage{amsmath,amssymb,enumerate}
\usepackage{amsthm}
\usepackage[noadjust]{cite}
\usepackage{setspace}
\usepackage{booktabs}
\usepackage[caption=false,font=footnotesize]{subfig}

\newtheorem{theorem}{Theorem}

\newtheorem{proposition}[theorem]{Proposition}
\newtheorem{lemma}[theorem]{Lemma}
\newtheorem{remark}{Remark}

\newtheorem{assumption}{Assumption}

\newcommand{\diag}{\mathop{\mathrm{diag}}}
\newcommand{\rad}{\mathop{\mathrm{rad}}}
\newcommand{\m}{\mathop{\mathrm{m}}}

\usepackage{flushend}

\begin{document}

\title{Warped Gaussian Processes Occupancy Mapping with Uncertain Inputs}

\author{Maani Ghaffari Jadidi, Jaime Valls Miro, and Gamini Dissanayake
\thanks{The authors are with Centre for Autonomous System, Faculty of Engineering and IT, University of Technology Sydney, Ultimo, NSW 2007, Australia {\tt\small \{maani.ghaffarijadidi, jaime.vallsmiro, gamini.dissanayake\}@uts.edu.au}}
% \thanks{Manuscript received April 19, 2005; revised August 26, 2015.}
}

% Make room for more info lines in the \author command  
% \author{Maani Ghaffari Jadidi, Jaime Valls Miro, and Gamini Dissanayake%
% \thanks{Manuscript received: August, 31, 2016; Revised November, 22, 2016; Accepted December, 19, 2016.}%Use only for final RAL version
% \thanks{This paper was recommended for publication by Editor Cyrill Stachniss upon evaluation of the Associate Editor and Reviewers' comments.}%Use only for final RAL version
% \thanks{The authors are with the Centre for Autonomous System, Faculty of Engineering and IT, University of Technology Sydney, Ultimo, NSW 2007, Australia {\tt\small \{maani.ghaffarijadidi, jaime.vallsmiro, gamini.dissanayake\}@uts.edu.au}}%
% \thanks{Digital Object Identifier (DOI): see top of this page.}
% }
%Use only for final RAL version. 

% Paper headers 
% \markboth{IEEE Robotics and Automation Letters. Preprint Version. Accepted December, 2016}
% {Ghaffari Jadidi \MakeLowercase{\textit{et al.}}: Warped Gaussian Processes Occupancy Mapping with Uncertain Inputs}  
% Use only for final RAL version

\maketitle

%%%%%%%%%%%%%%%%%%%%%%%%%%%%%%%%%%%%%%%%%%%%%%%%%%%%%%%%%%%%%%%%%%%%%%%%%%%%%%%%
\begin{abstract}
In this paper, we study extensions to the Gaussian Processes (GPs) continuous occupancy mapping problem. There are two classes of occupancy mapping problems that we particularly investigate. The first problem is related to mapping under pose uncertainty and how to propagate pose estimation uncertainty into the map inference. We develop expected kernel and expected sub-map notions to deal with uncertain inputs. In the second problem, we account for the complication of the robot's perception noise using Warped Gaussian Processes (WGPs). This approach allows for non-Gaussian noise in the observation space and captures the possible nonlinearity in that space better than standard GPs. The developed techniques can be applied separately or concurrently to a standard GP occupancy mapping problem. According to our experimental results, although taking into account pose uncertainty leads, as expected, to more uncertain maps, by modeling the nonlinearities present in the observation space WGPs improve the map quality.
\end{abstract}


%%%%%%%%%%%%%%%%%%%%%%%%%%%%%%%%%%%%%%%%%%%%%%%%%%%%%%%%%%%%%%%%%%%%%%%%%%%%%%%%
\IEEEpeerreviewmaketitle

\section{INTRODUCTION}

\IEEEPARstart{I}{n} many scenarios such as robotic navigation, the robot pose is partially observable; and we have access only to an estimate (noise-corrupted version) of the robot pose, as depicted in Figure~\ref{fig:intel_posecov}. Under these circumstances, the robot requires to navigate in an uncertain environment~\cite{roy1999coastal,Prentice01112009,valencia2013planning,vallve2015potential}, and the probability distribution of the robot pose will be the input for the mapping problem. In practice, and based on the application, most of the occupancy mapping techniques ignore robot pose uncertainty for map representation; either for efficiency or as the resultant map is not suitable for navigation. Furthermore, dense representation of the state often makes uncertainty propagation intractable. This problem is not unique to \emph{Gaussian Processes Occupancy Maps} (GPOMs), but it is also present in \emph{Occupancy Grid Maps}. With this motivation, we study the problem of GP occupancy mapping under pose uncertainty. The first solution is uncertainty propagation through \textit{kernel} functions. The second solution we propose uses the \textit{expected sub-map} notion to incorporate pose uncertainties into the map building process.

The second problem studied is motivated by the fact that due to the smoothness common in the resultant regressed maps, inferring a high-quality map compatible with the actual shape of the environment is non-trivial~\cite{jadidi2016gaussian}. 
Furthermore, for a complicated task such as robotic mapping~\cite{Thrun:2003:RMS:779343.779345}, the additive white Gaussian noise assumption in standard GPs can be simplistic. To account for these problems, we improve the incremental GPOM technique using \emph{Warped Gaussian Processes} (WGPs)~\cite{snelson2004warped}. The core idea is to map the target values through a warping (transforming) function to capture the nonlinear behavior of the observations.

We tackle the mentioned problems to improve map quality and provide results using incremental WGP Occupancy Maps (WGPOMs) under pose uncertainty.

\begin{figure}
  \centering 
  \includegraphics[width=0.83\columnwidth,trim={3.cm 3cm 2.25cm 2.75cm},clip]{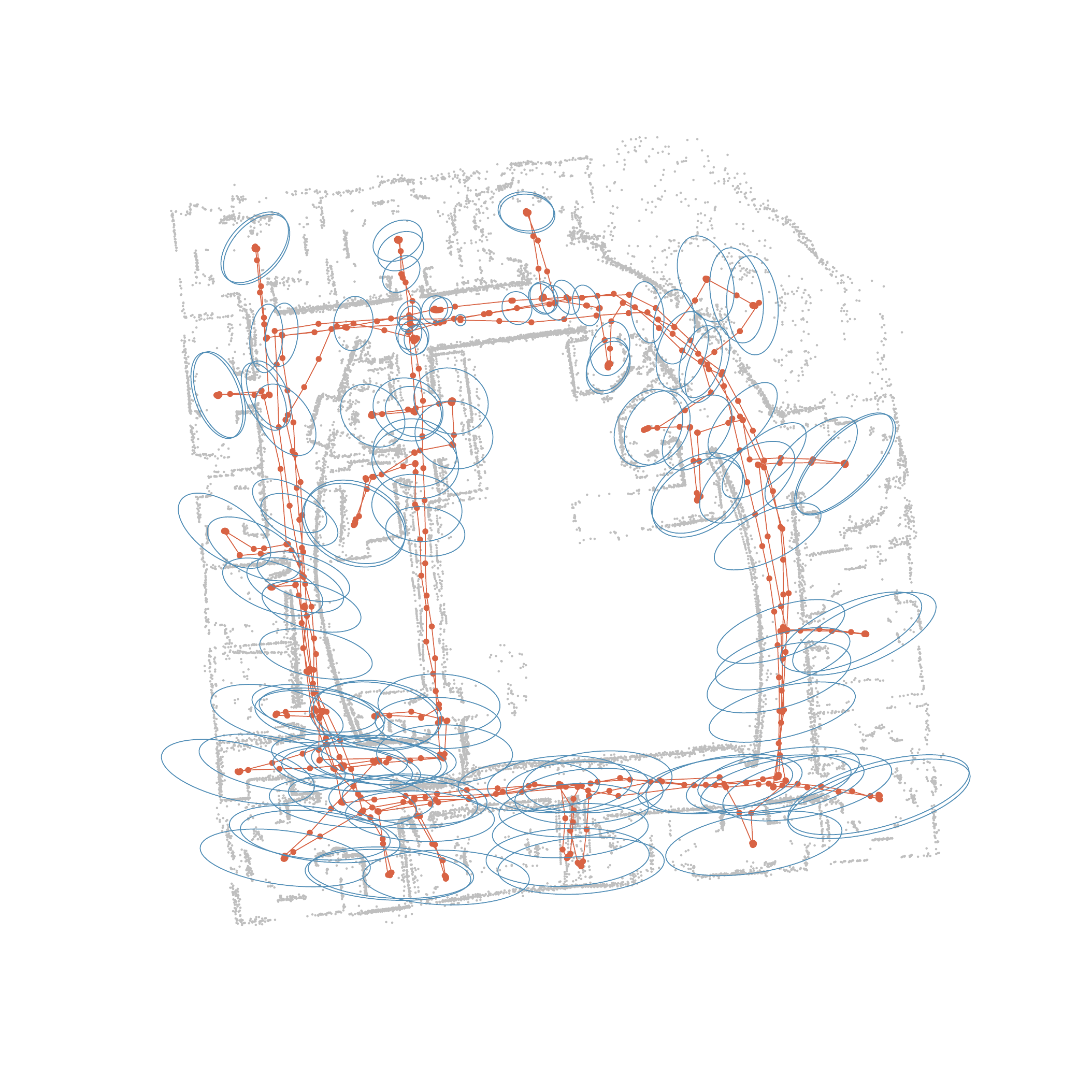}
  \caption{The pose-graph estimation of the robot trajectory from the Intel research lab. dataset~\cite{Radish_data_set}, solved using the techniques in~\cite{ila2010information,valencia2013planning}. The large uncertainties from pose estimation are often ignored in the dense occupancy map representation. For the sake of clarity, loop-closures are omitted, and only one-sixth of the pose covariances are illustrated.}
  \label{fig:intel_posecov}
\end{figure}

\subsection*{Notation}
Probabilities and probability densities are not distinguished in general. Matrices are capitalized in bold, such as in $\boldsymbol X$, and vectors are in lower case bold type, such as in $\boldsymbol x$. Vectors are column-wise and $1\colon n$ means integers from $1$ to $n$. The Euclidean norm is shown by $\lVert \cdot \rVert$. $\mathrm{tr}(\boldsymbol X)$ and $\lvert \boldsymbol X \rvert$ denote the trace and the determinant of matrix $\boldsymbol X$, respectively. Random variables, such as $X$, and their realizations, $x$, are sometimes denoted interchangeably. $x^{[i]}$ denotes a reference to the $i$-th element of the variable. An alphabet such as $\mathcal{X}$ denotes a set. A reference to a test set quantity is shown by $\boldsymbol x_*$. The $n$-by-$n$ identity matrix is denoted by $\boldsymbol I_{n}$.  The sign function, and the absolute value are denoted by $\mathrm{sgn}(x)$, and $\lvert x \rvert$, respectively. Finally, $\mathbb{E}[\cdot]$ and $\mathbb{V}[\cdot]$ denote the expected value and variance of a random variable, respectively.

\subsection*{Outline}
A review of related works is given in the following section. We first discuss the problem of Gaussian processes occupancy mapping under pose uncertainty in Section~\ref{sec:uncertainpose}; followed by presenting the warped Gaussian processes occupancy mapping in Section~\ref{sec:Warped}. Robotic mapping experiments are presented in Section~\ref{sec:wgpomres}, and finally, Section~\ref{sec:conclusion} concludes the paper.

\section{Related Work}

Current robotic navigation algorithms rely on a dense environment representation for safe navigation. Occupancy grid maps are the standard way of environment representation in robotics and are the natural choice for online implementations~\cite{moravec1985high,elfes1987sonar,thrun2003learning,hornung2013octomap,merali2014optimizing}. However, they simplify the mapping problem to a set of independent cells and estimate the probability of occupancy for each cell by ignoring the correlation between them. The map posterior is then approximated as the product of its marginals.

The FastSLAM algorithm~\cite{Montemerlo03fastslam2} exploits \emph{Rao-Blackwellized particle filters}~\cite{doucet2000rao}, and relies on the fact that, given the robot trajectory, landmarks are conditionally independent. In practice, FastSLAM can generate accurate maps; however, it suffers from degeneracy and depletion problems~\cite{bailey2006consistencypf}. Full SLAM methods estimate the \emph{smoothing distribution} of the robot trajectory and the map, given all available data. In robotics, full SLAM using probabilistic graphical models representation and the \emph{maximum a posteriori} estimate derived by solving a nonlinear least squares problem~\cite{dellaert2006square,grisetti2008online,kaess2011isam2,konolige2010efficient,thrun2006graph} can be considered the state-of-the-art and we thus focus our attention on using a pose-graph estimation of the robot trajectory as in~\cite{ila2010information,valencia2013planning}.

Kernel methods in the form of Gaussian processes (GPs) framework~\cite{rasmussen2006gaussian} are non-parametric regression and classification techniques that have been used to model spatial phenomena~\cite{vasudevan2009gaussian}. The development of GPOM is discussed in~\cite{o2009contextual,t2012gaussian}. The incremental GP map building using the Bayesian Committee Machine (BCM) technique~\cite{tresp2000bayesian} is developed in~\cite{kim2012building,jadidi2013exploration,jadidi2013acra,maani2014com} and for online applications in~\cite{7487232}. In~\cite{ramoshilbert}, the Hilbert maps technique is proposed and is more scalable. However, it is an approximation for continuous occupancy mapping and produces maps with less accuracy than GPOM.

Approximate methods for uncertainty propagation into GPs models through kernel functions are proposed in~\cite{girard2004approximate}, and developed for GPOM in~\cite{o2010contextual}. 
Generally speaking, training and query points can both be noisy. In~\cite{girard2004approximate}, the problem of prediction at an uncertain input is discussed while it is assumed the input in training data is noise free. In~\cite{o2010contextual}, using a similar approach, the idea is extended to account for noisy training input. Similarly, we assume the input in training data is uncertain and query points are deterministic. In this paper, we employ this technique and develop the expected sub-map technique for approximate uncertainty propagation, then incorporate both techniques into the WGPOM framework.

\section{Mapping under Pose Uncertainty}
\label{sec:uncertainpose}
The main challenge in building occupancy maps under robot pose uncertainty is the dense representation of map belief which makes uncertainty propagation computationally expensive. Maximum likelihood dense map representations are currently the common practice which does not necessarily produce correct maps, especially if pose estimation uncertainties are significant. This popularity can be understood from the fact that employing an environment representation constructed with significant uncertainties results in vague obstacles and free space and is not suitable for robotic motion planning and navigation. However, accounting for pose uncertainties in mapping is not only important for correct map representations, but also for motion planning (prediction) tasks.

\subsection{Problem Statement and Formulation}
\label{subsec:problemuin}
Let $\mathcal{M}$ be the set of possible static occupancy maps. We consider the map of the environment as an $n_m$-tuple random variable \mbox{$(M^{[1]},\dots,M^{[n_m]})$} whose elements are described by a normal distribution \mbox{$M^{[i]} \sim \mathcal{N}(\mu^{[i]},\sigma^{[i]})$, $\forall \ i \in \{1\colon n_m\}$}. Let $\mathcal{X} \subset \mathbb{R}^2$ be the set of spatial coordinates to build a map on. Let \mbox{$y = \tilde{y} + \epsilon_y$} be a noisy measurement (class label; $-1$ and $1$ for unoccupied and occupied, respectively) at a noisy sample \mbox{$\boldsymbol x = \boldsymbol{\tilde{x}} + \boldsymbol\epsilon_x$}, where \mbox{$\epsilon_y \sim \mathcal{N}(0,\sigma_n^2)$} and \mbox{$\boldsymbol\epsilon_x \sim \mathcal{N}(\boldsymbol 0,\boldsymbol\Sigma_x)$}. Define a training set \mbox{$\mathcal{D} = \{(\boldsymbol x^{[i]},y^{[i]}) \mid i=1\colon n_t\}$} which consists of noisy measurements at noisy locations. Let function \mbox{$f:\mathcal{X}\rightarrow\mathcal{M}$}, i.e. \mbox{$y = f(\boldsymbol{\tilde{x}} + \boldsymbol\epsilon_x) + \epsilon_y$}, be the real underlying process that we model as a Gaussian process \mbox{$f(\boldsymbol x) \sim \mathcal{GP}(0, k(\boldsymbol x,\boldsymbol x'))$}, where \mbox{$k:\mathcal{X} \times \mathcal{X} \rightarrow \mathbb{R}$} is the \emph{covariance function} or \emph{kernel}; and $\boldsymbol x$ and $\boldsymbol x'$ are either in the training or the test (query) sets.
Estimate \mbox{$p(M=m\mid \mathcal{D})$}, i.e. the map posterior probability given a noisy training set. For a given query point in the map, $\boldsymbol x_*$, GP predicts a mean, $\mu$, and an associated variance, $\sigma$. We can write \mbox{$m^{[i]} = y(\boldsymbol x^{[i]}_*) \sim \mathcal{N}(\mu^{[i]},\sigma^{[i]})$}.
To show a valid probabilistic representation of the map $p(m^{[i]}\mid \mathcal{D})$, the classification step squashes data into the range $[0,1]$.

We propose two methods to solve the defined problem. The first approach is based on the \textit{expected kernel}. The alternative approach, \textit{expected sub-map}, treats all inputs deterministically and propagates pose uncertainties through uncertain map fusion. The following assumptions are made in the present work:
\begin{assumption}[Deterministic query points]
 In the problem of Gaussian processes occupancy mapping under robot pose uncertainty, query points are deterministic.
\end{assumption}
\begin{assumption}
\label{assump:statcov}
 The covariance function in the expected sub-map method is stationary, $k(\boldsymbol x, \boldsymbol x')=k(\lVert \boldsymbol x - \boldsymbol x' \rVert)$.
\end{assumption}
\begin{remark}
 Using Assumption~\ref{assump:statcov}, map inference in the local coordinates of the robot (local map) can be done using deterministic inputs.
\end{remark}

\subsection{System Dynamics}
\label{subsec:sysdayn}
The equation of motion of the robot is governed by the nonlinear partially observable equation as follows.
\begin{equation}
\small
\label{eq:reom}
\boldsymbol x_{t+1}^- = f(\boldsymbol x_{t}, \boldsymbol u_{t}, \boldsymbol w_{t}) \quad \boldsymbol w_{t} \sim \mathcal{N}(\boldsymbol 0,\boldsymbol Q_{t})
\end{equation}
moreover, with appropriate linearization at the current state estimate, we can predict the state covariance matrix as
\begin{equation}
\small
\label{eq:predcov}
 \boldsymbol \Sigma_{t+1}^- = \boldsymbol F_t \boldsymbol \Sigma_{t} \boldsymbol F_t^T + \boldsymbol W_t \boldsymbol Q_t \boldsymbol W_t^T
\end{equation}
where $\boldsymbol F_t = \frac{\partial f}{\partial \boldsymbol x} \vert_{\boldsymbol x_{t}, \boldsymbol u_{t}}$ and $\boldsymbol W_t = \frac{\partial f}{\partial \boldsymbol w} \vert_{\boldsymbol x_{t}, \boldsymbol u_{t}}$ are the Jacobian matrices calculated with respect to $\boldsymbol x$ and $\boldsymbol w$, respectively.

\subsection{Expected Kernel}
\label{subsec:ukernel}

The core idea in the expected kernel approach is taking an expectation of the covariance function over uncertain inputs. Let $\boldsymbol x$ be distributed according to a probability distribution $p(\boldsymbol x)$. The expected covariance function can be computed as
\begin{equation}
 \small
 \label{eq:expctkernel}
  \tilde{k} = \mathbb{E}[k] = \int_{\Omega} kdp
\end{equation}
In general, this integral is analytically intractable; therefore we employ two numerical approximations to solve~\eqref{eq:expctkernel}. However, for the case of the squared exponential (SE) kernel, a closed-form solution exists~\cite{NIPS2002_2313}. Hence, once the expected covariance matrix is calculated, we can compute the predictive conditional distribution for a single query point similar to standard GPs.

\subsubsection*{Monte Carlo Integration}
\label{subsubsec:montecarlo}
Since we assume the distribution of the uncertain input is known, by drawing independent samples, $\boldsymbol x^{[i]}$, from $p(\boldsymbol x)$ and using a Monte-Carlo technique, we can approximate Equation~\eqref{eq:expctkernel} by
\begin{equation}
 \small
 \label{eq:mckernel}
 \tilde{k} = \frac{1}{n}\sum_{i=1}^{n}k_i
\end{equation}
where $k_i$ is the covariance function computed at $\boldsymbol x^{[i]}$.
\begin{remark}
 In Equation~\eqref{eq:mckernel}, the covariance $k(\boldsymbol x,\boldsymbol x)$ and cross-covariance $k(\boldsymbol x,\boldsymbol x_*)$ are both denoted as $k_i$. Depending on the input, the integration is only performed on training points as it is assumed query points are deterministic.
\end{remark}

\subsubsection*{Gauss-Hermite Quadrature}
\label{subsubsec:gausshermit}
Gauss-Hermite quadrature~\cite{davis2007methods} of integrals of the kind $\int_{-\infty}^{\infty} \exp(-x^2)f(x)dx$ are given by $\int_{-\infty}^{\infty} \exp(-x^2)f(x)dx \approx \sum_{j=1}^n w^{[j]} f(x^{[j]})$.
The multi-variate normal distribution of noisy input is given by $\mathcal{N}(\boldsymbol{\tilde{x}}, \boldsymbol\Sigma_x)$. Through a change of variable such that $\boldsymbol L \boldsymbol L^{T} = 2\boldsymbol\Sigma_x$ and $\boldsymbol u = \boldsymbol L^{-1}(\boldsymbol x - \boldsymbol{\tilde{x}})$, where $\boldsymbol L$ is a lower triangular matrix that can be calculated using a Cholesky factorization, the Equation~\eqref{eq:expctkernel} can be approximated as
\begin{equation}
 \small
\label{eq:multighk}
 \tilde{k} = (2\pi)^{\frac{-d}{2}} \sum_{i_1=1}^{n} ... \sum_{i_d=1}^{n} \bar{w} k_{i_{1:d}}
\end{equation}
where $\bar{w} \triangleq \prod_{j=1}^{d} w^{[i_{j}]}$, $u^{[i_{j}]}$ are the roots of the Hermite polynomial $H_n$, $\boldsymbol u_{i_{1:d}} \triangleq [u^{[i_{1}]},...,u^{[i_{d}]}]^T$, and $k_{i_{1:d}}$ is the covariance function computed at $\boldsymbol x_{i_{1:d}} = \boldsymbol L \boldsymbol u_{i_{1:d}} + \boldsymbol{\tilde{x}}$. When $d=2$, we can simplify Equation~\eqref{eq:multighk} and write \mbox{$\tilde{k} = \frac{1}{2\pi} \sum_{i_1=1}^{n}\sum_{i_2=1}^{n} \bar{w} k_{i_{1:2}}$}.
\begin{remark}
 We assumed points from the map spatial support, $\boldsymbol x \in \mathcal{X}$, are global coordinates. In practice, to transform local points, an unscented transform~\cite{julier1997new} is used to reduce linearization errors~\cite{o2010contextual}.
\end{remark}

% \begin{example}
 An illustrative example of GP regression where inputs are uncertain is shown in Figure~\ref{fig:ek_toy_eg}. By propagating the input uncertainty using the expected kernel, the output does not follow the observations exactly yet remains consistent as the underlying function is within the estimated uncertainty bounds.
% \end{example}

\begin{figure}
  \centering 
  \includegraphics[width=.85\columnwidth,trim={0.75cm 0cm 1.5cm 0.5cm},clip]{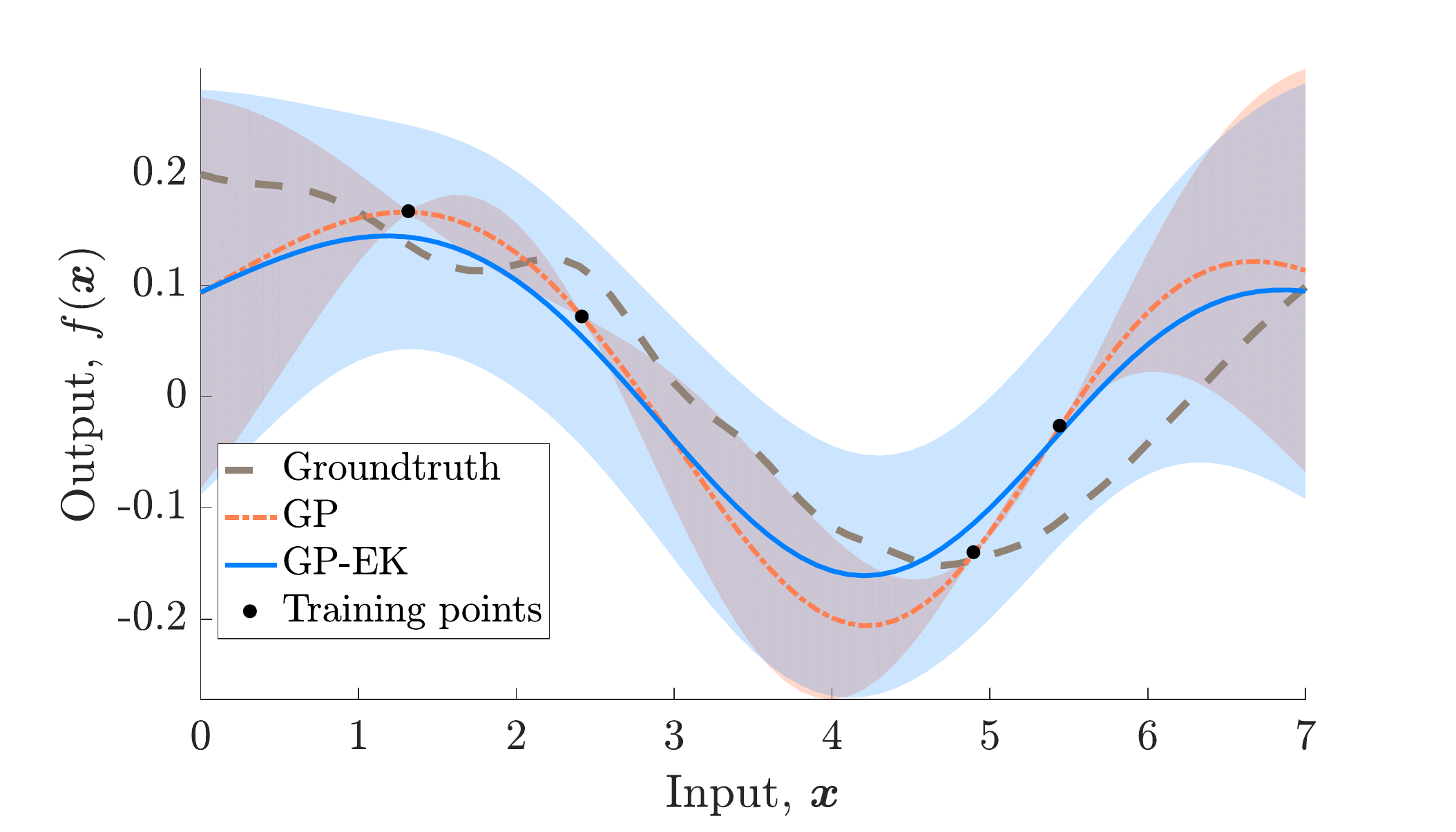}
  \caption{The plot shows an example of GP regression with uncertain inputs. The GP-EK shows GP regression by incorporating the input uncertainty using the expected kernel. The standard GP results are generated by ignoring the input uncertainty which cannot provide a consistent solution. The groundtruth function is \mbox{$f(x) = \frac{1}{5} \cos(x^2) e^{-x} + \frac{3}{20} \sin(x)$} and training points are corrupted by $\epsilon_x \sim \mathcal{N}(0,0.6^2)$.}
  \label{fig:ek_toy_eg}
\end{figure}

\subsection{Expected Sub-map}
\label{sec:Post}
We exploit the fact that a stationary covariance function does not depend on the selected coordinates, i.e. the global or a local sub-map frame. Therefore, we treat all training inputs as noise free and conduct map inference using deterministic inputs in the local (sensor/robot) frame. To fuse the inferred sub-map into the global map, we draw independent samples from $p(\boldsymbol x_{t})$. In other words, by taking the expectation over the location of the sub-map, we propagate uncertainty of each map point to its neighborhood. Thus, we have \mbox{$p(m|\boldsymbol{\tilde{x}},y) = \int p(m|\boldsymbol x,y) p(\boldsymbol x_{t}) d\boldsymbol x_{t}$}, and by drawing independent samples from $p(\boldsymbol x)$ and using a Monte-Carlo approximation it follows that \mbox{$p(m|\boldsymbol{\tilde{x}},y) \approx \frac{1}{n}\sum_{j=1}^{n} p(m_j|\boldsymbol x_j,y)$}. Note that any sub-map $p(m_j|\boldsymbol x_j,y)$ can be fused into the global map using the algorithms in~\cite{jadidi2016gaussian}. However, as a result of sampling, the expected map, $p(m|\boldsymbol{\tilde{x}},y)$, is similar to a mixture distribution; therefore, the mean and variance calculations need to be addressed accordingly. We present the following proposition to calculate the first two moments of $p(m|\boldsymbol{\tilde{x}},y)$.
\begin{lemma}
\label{lem:meanvarmix}
Let $X_1,X_2,...,X_n$ be random variables that are distributed according to probability densities $p(x_1), p(x_2),...,p(x_n)$, with constant weights $w_1,w_2,...,w_n$, where $\sum_{i=1}^n w_i = 1$. The probability density function of the mixture is $p(x) = \sum_{i=1}^n  w_i p(x_i)$. Given $\mu_i = \mathbb{E}[X_i]$ and $\sigma_i^2 = \mathbb{V}[X_i]$, the mean and variance of the mixture density is given by $\mu = \mathbb{E}[X] = \sum_{i=1}^n  w_i \mathbb{E}[X_i]$ and $\sigma^2 = \mathbb{V}[X] = \sum_{i=1}^n  w_i (\sigma_i^2 + \mu_i^2) - (\sum_{i=1}^n  w_i \mu_i)^2$.
\end{lemma}
\begin{proof}
 The proof follows from the fact that for the $k$-th moment of the mixture, we can write \mbox{$\mathbb{E}[X^{(k)}] = \sum_{i=1}^n w_i \mathbb{E}[X_i^{(k)}]$} and define the variance accordingly.
\end{proof}
\begin{proposition}[Expected sub-map fusion]
 In incremental map building, to compute $p(m|\boldsymbol{\tilde{x}},y)$, sampled sub-maps can be fused into the global map using the following equations
 \begin{equation}
 \small
 \label{eq:expmapmean}
  \mathbb{E}[M] = \frac{1}{n}\sum_{j=1}^n \mathbb{E}[M_j]
 \end{equation}
 \begin{equation}
 \small
 \label{eq:expmapvar}
  \mathbb{V}[M] = \frac{1}{n} (\sum_{j=1}^n (\mathbb{V}[M_j] + \mathbb{E}[M_j]^2) - \frac{1}{n} (\sum_{j=1}^n \mathbb{E}[M_j])^2)
 \end{equation}
 where $\mathbb{E}[M_j]$ is the updated global map built using the $j$-th independently drawn robot pose sample, and \eqref{eq:expmapmean} and \eqref{eq:expmapvar} can be computed point-wise for every map point $m^{[i]}$.
\end{proposition}
\begin{proof}
 The proof directly follows from Lemma~\ref{lem:meanvarmix}.
\end{proof}

Therefore, we can perform incremental map fusion by considering the robot pose uncertainty without modifying the GP framework. 

\section{Warped GP Occupancy Mapping}
\label{sec:Warped}
The primary challenge in modeling the environment ``accurately'' is the different nature of free and occupied classes. Free space tends to span vast areas while occupied space often represents the structural shape of the environment. In addition, the assumption of additive Gaussian noise in the observations in standard GPs is unable to capture complexity in observations appropriately. We propose to employ Warped Gaussian Processes to account for the nonlinear behavior of observations. This method is appealing as it allows for non-Gaussian noise in the observation space. However, exact inference is not possible anymore, and approximate inference algorithms such as \emph{expectation propagation}~\cite{minka2001family} or \emph{variational Bayes}~\cite{jordan1999introduction} are required.

The idea to accommodate non-Gaussian distributions and noise is to use a nonlinear monotonic function for warping (transforming) the observation space~\cite{snelson2004warped}. Let $g_w(\cdot)$ be a transformation from the observation space to a latent space as
\begin{equation}
\small
 \label{eq:warptrans}
 t^{[i]} = g_w(y^{[i]}; \boldsymbol\psi) \qquad i=1\colon n
\end{equation}
where $\small{\boldsymbol\psi}$ denotes the vector of warping function hyperparameters and $\small{\boldsymbol t = [t^{[1]},\dots,t^{[n]}]^T}$ is the vector of latent targets. Now we can re-write the GP formulation for the latent target and by accounting for the transformation between a true observation and the latent target, the negative log of the marginal likelihood (NLML) can be written as
\begin{equation}
\footnotesize
\label{warpednlml}
\begin{split}
	&\log\ p(\boldsymbol y|X,\boldsymbol\theta,\boldsymbol\psi) = -\frac{1}{2}g_w(\boldsymbol y)^{T} [\boldsymbol K(\boldsymbol X,\boldsymbol X)+\sigma_n^2 \boldsymbol I_{n}]^{-1} g_w(\boldsymbol y) \\& -\frac{1}{2}\log\ \arrowvert \boldsymbol K(\boldsymbol X,\boldsymbol X)+\sigma_n^2 \boldsymbol I_{n} \arrowvert-\frac{n}{2}\log\ 2\pi + \sum_{i=1}^{n}\log\frac{\partial g_w(y)}{\partial y}\Bigg|_{y^{[i]}}
\end{split}
\end{equation}
in which the last term is the Jacobian of the defined transformation. To compute the mean at a new test point, it is possible to calculate the expectation of the inverse warping function over the latent target predictive density, therefore
\begin{equation}
\small
 \label{eq:expinv}
 \mathbb{E}[y^{[n+1]}] = \int g_w^{-1}(t)\mathcal{N}(\hat{t}^{[n+1]},\sigma^{[n+1]}) dt = \mathbb{E}[g_w^{-1}]
\end{equation}
This integral can be computed numerically using Gauss-Hermite quadrature with a weighted sum of the inverse warping function $g_w^{-1}$.

\begin{figure}
  \centering 
  \includegraphics[width=.85\columnwidth,trim={1.25cm 0cm 1.5cm 0.5cm},clip]{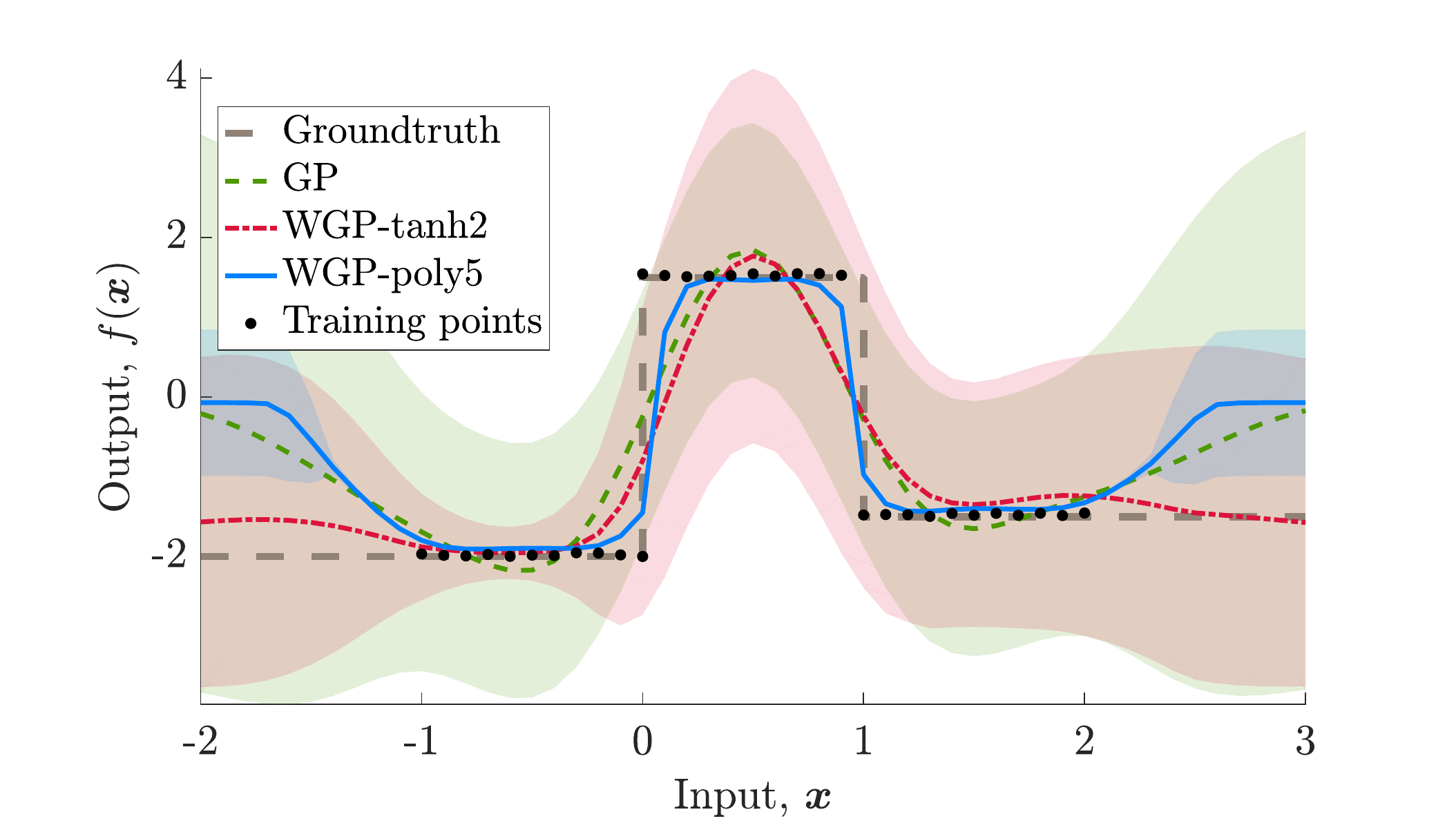}
  \caption{A challenging example of regression using standard and Warped GPs. The measurements are sampled from the true output values by adding noise, i.e. $\epsilon_y \sim \mathcal{N}(0,0.05^2)$. The warping functions are $\tanh (\ell=2)$ and polynomial of degree five. The $\tanh$ demonstrate a better extrapolating behavior, while the polynomial function can follow the underlying function more closely. However, polynomials are prone to over-fitting as it can be seen that the estimated uncertainties (blue shaded region) are significantly narrower.}
  \label{fig:toy_eg}
\end{figure}

Inspired by the neural network transfer functions, a sum of hyperbolic tangent functions satisfies the requirements for the transformation to be monotonic and at the same time allowing for complicated mappings. With hyperparameters vector $\small{\boldsymbol\psi = [\boldsymbol a, \boldsymbol b, \boldsymbol c]^T}$, the function can be defined as
\begin{equation}
\footnotesize
 \label{eq:warpfun}
 g_w(y; \boldsymbol\psi) = y + \sum_{i=1}^{\ell} a^{[i]} \tanh (b^{[i]}(y + c^{[i]})) \quad a^{[i]}, b^{[i]} \geq 0 \ \forall i \in \{1:\ell\}
\end{equation}
where the parameter $\ell$ is the number of steps and has to be set depending on the complexity of observations, \mbox{$\small{\boldsymbol a = [a^{[1]},\dots,a^{[\ell]}]}$}, $\small{\boldsymbol b = [b^{[1]},\dots,b^{[\ell]}]}$, and $\small{\boldsymbol c = [c^{[1]},\dots,c^{[\ell]}]}$. Alternative warping functions can be polynomials ($\small{\boldsymbol\psi = \boldsymbol c}$):
\begin{equation}
\small
 \label{eq:warppoly}
 g_w(y; \boldsymbol\psi) = y + \sum_{i=2}^{\ell} c^{[i-1]} \mathrm{sgn}(y) \lvert y \rvert^i  \quad c^{[i]} \geq 0 \ \forall i  \in \{2:\ell\}
\end{equation}

% \begin{example}
 Figure~\ref{fig:toy_eg} shows a simple yet challenging example for regression using standard and Warped GPs. The measurements are corrupted by an additive Gaussian noise. Even though the noise is still Gaussian, the complicated structure of the underlying function makes modeling it non-trivial. Note that the inputs are deterministic.
% \end{example}
\begin{remark}
 By increasing the number of training points, it is possible to generate more accurate results using standard GPs. However, given the cubic time complexity of GPs, dense training datasets reduce the scalability of the algorithms significantly.
\end{remark}

\section{Results and Discussion}
\label{sec:wgpomres}
We now present results from experiments using a synthetic dataset and a real publicly available pose-graph dataset. The synthetic dataset, Figure~\ref{fig:synmap_setup}, is built in such a way as to highlight the strength of WGPOM to model complicated structural shapes and to better appreciate the mapping performance of the incremental GPOM and WGPOM under the expected kernel (EK) and expected sub-map (ESM) uncertainty propagation techniques. On the other hand, the Intel dataset~\cite{Radish_data_set}, as shown in Figure~\ref{fig:intel_posecov}, exposes an extreme real-world example of the problem at hand where highly uncertain robot poses along the estimated trajectory are present. 

We compare the overall taken time to build the entire map (using all the available data) as well as the map accuracy using the Area Under the receiving operating characteristic Curve (AUC)~\cite{fawcett2006introduction}. For each model, we learn the hyperparameters at the first increment of map building by minimization of the NLML using the first set of training data with manual supervision to ensure the best possible outcome for all models. The data processing and computations for the incremental map building are implemented using MATLAB.

\begin{figure}[!t]
  \centering  
    \includegraphics[width=.75\columnwidth,trim={1.75cm 0.5cm 1.75cm 1.cm},clip]{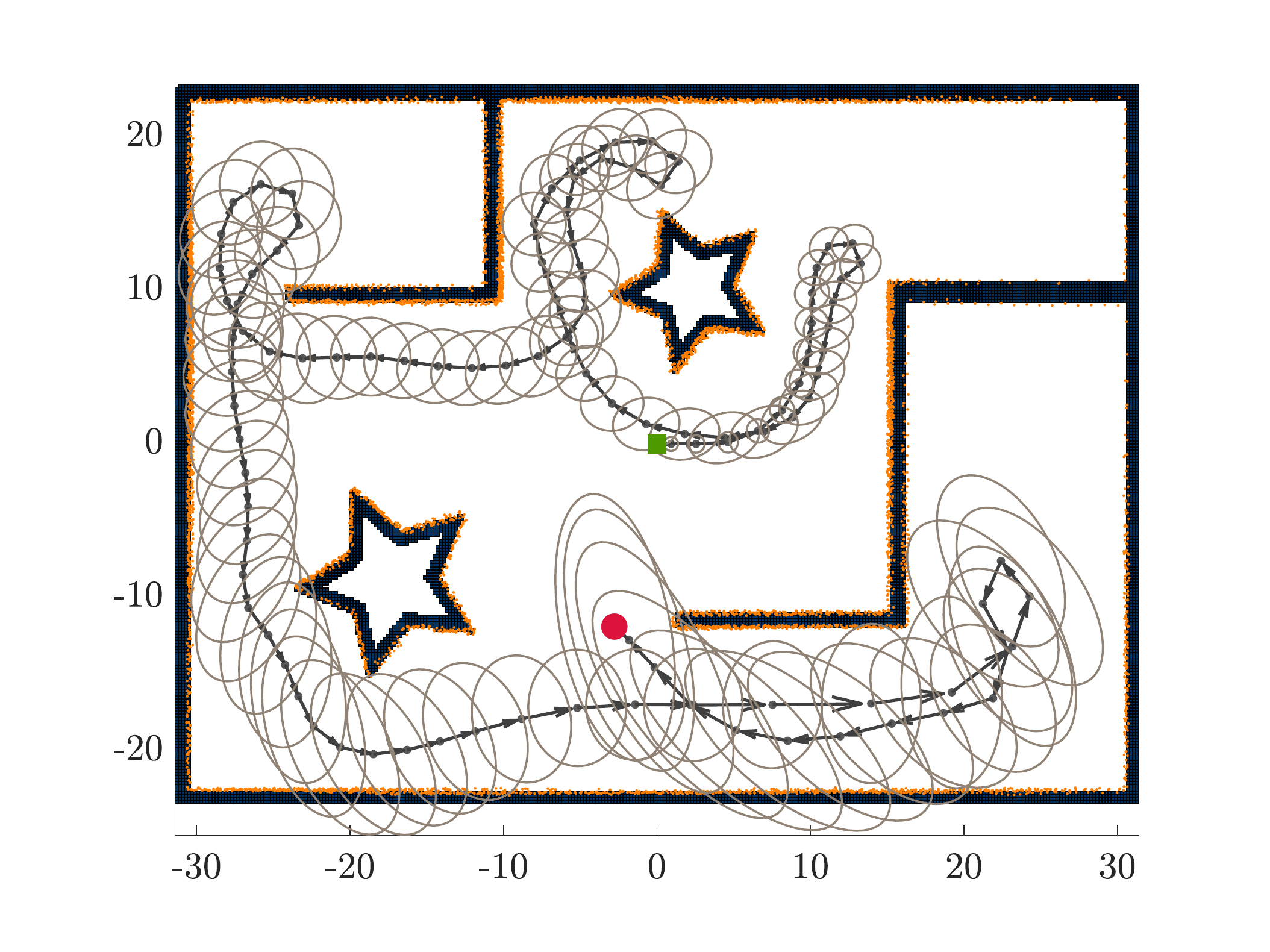}
  \caption{The synthetic dataset used for comparison of GPOM and WGPOM under various uncertainty propagation conditions. The figure shows collected observations along the robot trajectory (orange dots) and the robot position uncertainty ellipse at each corresponding pose ($\boldsymbol Q_3$ scenario). The robot starting position is shown in green. Map dimensions are in meters.}
  \label{fig:synmap_setup}
\end{figure}

\begin{figure}
  \centering  
  \subfloat{
    \includegraphics[width=.8\columnwidth,trim={0.5cm 0.cm 0.5cm 0.75cm},clip]{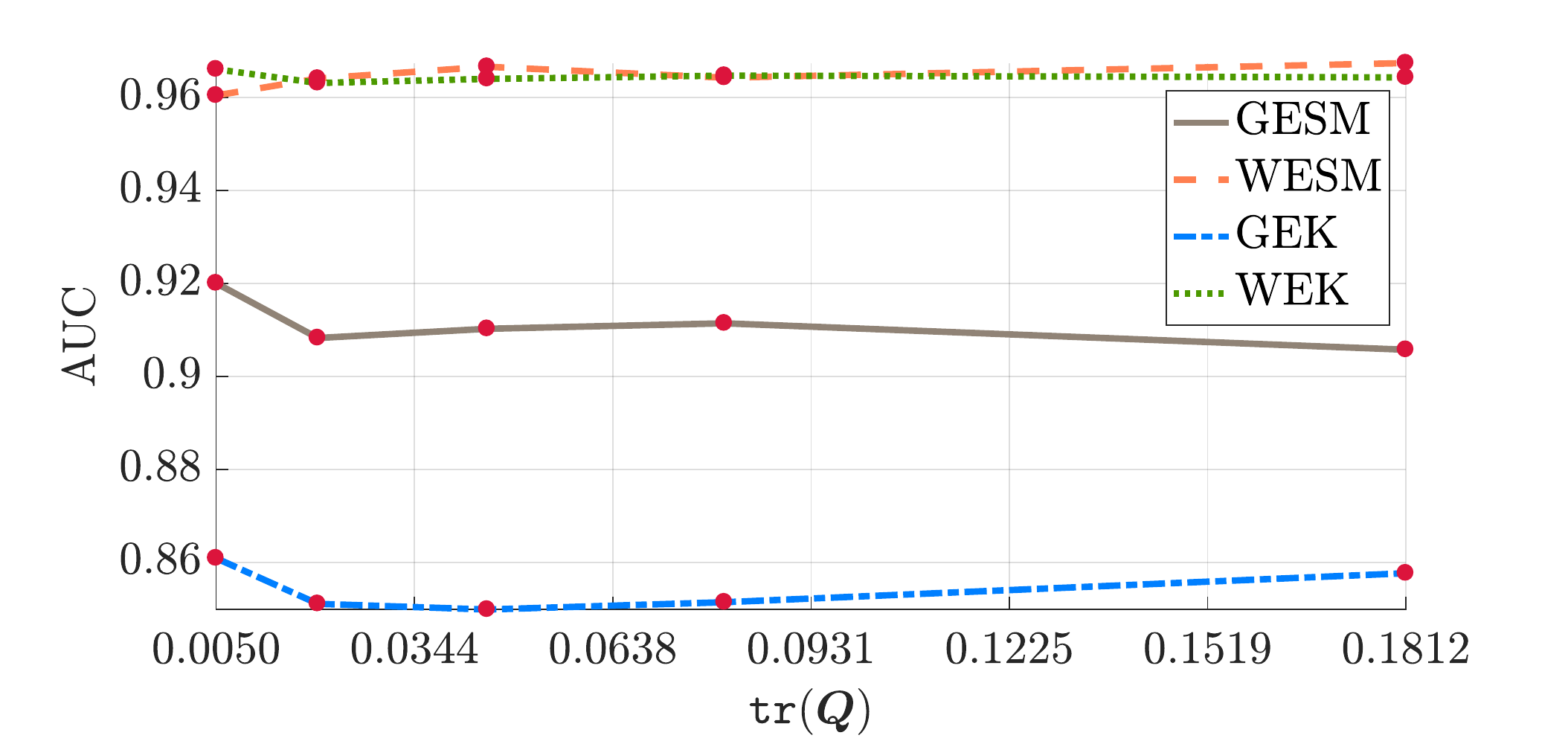}
    \label{fig:qplot}
    }\\
  \subfloat{
    \includegraphics[width=.8\columnwidth,trim={0.5cm 0.cm 0.5cm 0.75cm},clip]{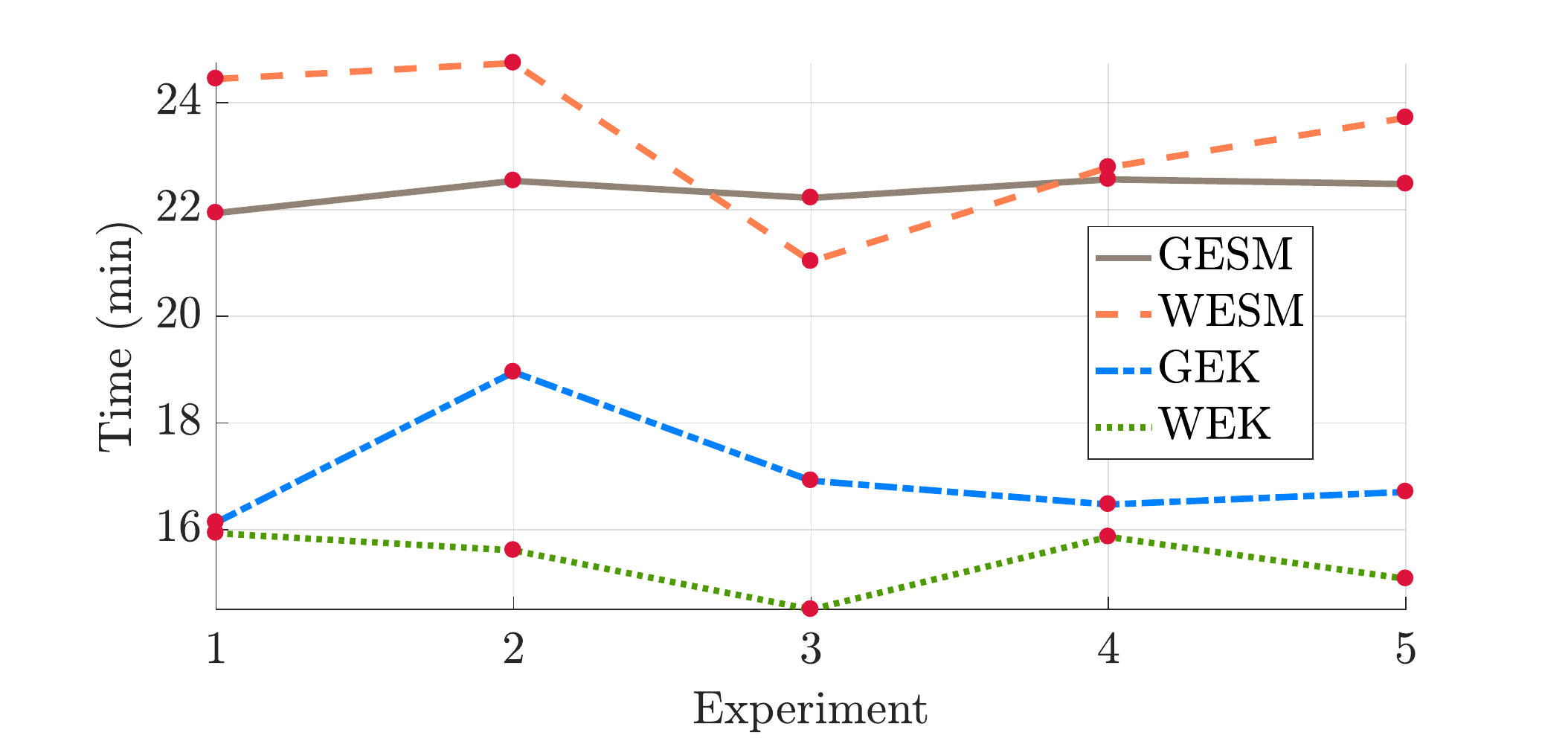}
    \label{fig:qplot_time}
    }
  \caption{The AUC and runtime for incremental GPOM and WGPOM using the synthetic dataset and proposed uncertainty propagation methods. The runtimes are in minutes. All maps are computed using $0.5\m$ resolution and a $k$d-tree data structure for the storage and nearest neighbor queries. As it is expected, generally, by propagating pose uncertainties into the map inference, the map quality degrades (GESM and GEK). However, applying WGPs can alleviate this effect and produce maps with improved accuracies (WESM and WEK).}
  \label{fig:synmap_qplot}
\end{figure}

\begin{table}[!ht]
\centering
\caption{Details of the experiments for the motion uncertainty effect in a synthetic dataset using the expected kernel and the expected sub-map uncertainty propagation schemes. The experiment is repeated five times by increasing the robot motion noise covariance, i.e. $\boldsymbol Q_1, \dots, \boldsymbol Q_5$. Note that all mapping techniques are incremental.}
\resizebox{\columnwidth}{!}{
\begin{tabular}{lll}
\toprule
\multicolumn{3}{l}{\textbf{Model selection:}} \\
Technique		& Cov. func.		& Warping func. \\ \midrule
GPOM			& Mat\'ern ($\nu = 5/2$)	& n/a			\\
WGPOM 			& SE (ARD)					& $\tanh$ ($\ell=2$)	\\ 
\multicolumn{3}{l}{\textbf{Compared mapping techniques:}} \\
Technique	& Uncertainty propagation		& Abbreviation \\ \midrule
GPOM		& EK			& GEK		\\
GPOM 		& ESM			& GESM		\\
WGPOM 		& EK			& WEK		\\ 
WGPOM 		& ESM			& WESM		\\ 
\multicolumn{3}{l}{\textbf{Numerical integration:}} \\
Uncertainty propagation			& Technique	& Samples \\ \midrule
EK					& Gauss-Hermite 	& $9$		\\
ESM 					& Monte-Carlo		& $10$		\\ 
\multicolumn{3}{l}{\textbf{Training and test (query) points size:}} \\
Parameter				& Symbol			& Value \\ \midrule
Tot. training points	& $n_{t}$			& 87833	\\ 
Ave. training points 	& $\bar{n}_{t}$		& 829	\\
Test points 		& $n_{q}$			& 6561	\\ 
\multicolumn{3}{l}{\textbf{Robot motion model noise covariances:}} \\
\multicolumn{3}{l}{$\boldsymbol Q_1 = \diag(0.05\m, 0.05\m, 0.25\rad)^2$, $\boldsymbol Q_2 = \diag(0.1\m, 0.1\m, 0.5\rad)^2$}	\\ 
\multicolumn{3}{l}{$\boldsymbol Q_3 = \diag(0.15\m, 0.15\m, 0.75\rad)^2$, $\boldsymbol Q_4 = \diag(0.2\m, 0.2\m, 1.00\rad)^2$}	\\ 
\multicolumn{3}{l}{$\boldsymbol Q_5 = \diag(0.3\m, 0.3\m, 2.00\rad)^2$}	\\ \bottomrule
\label{tab:starexpsetup}
\end{tabular}}
\end{table}

\begin{figure*}
  \centering  
  \subfloat[]{
    \includegraphics[width=.53\columnwidth]{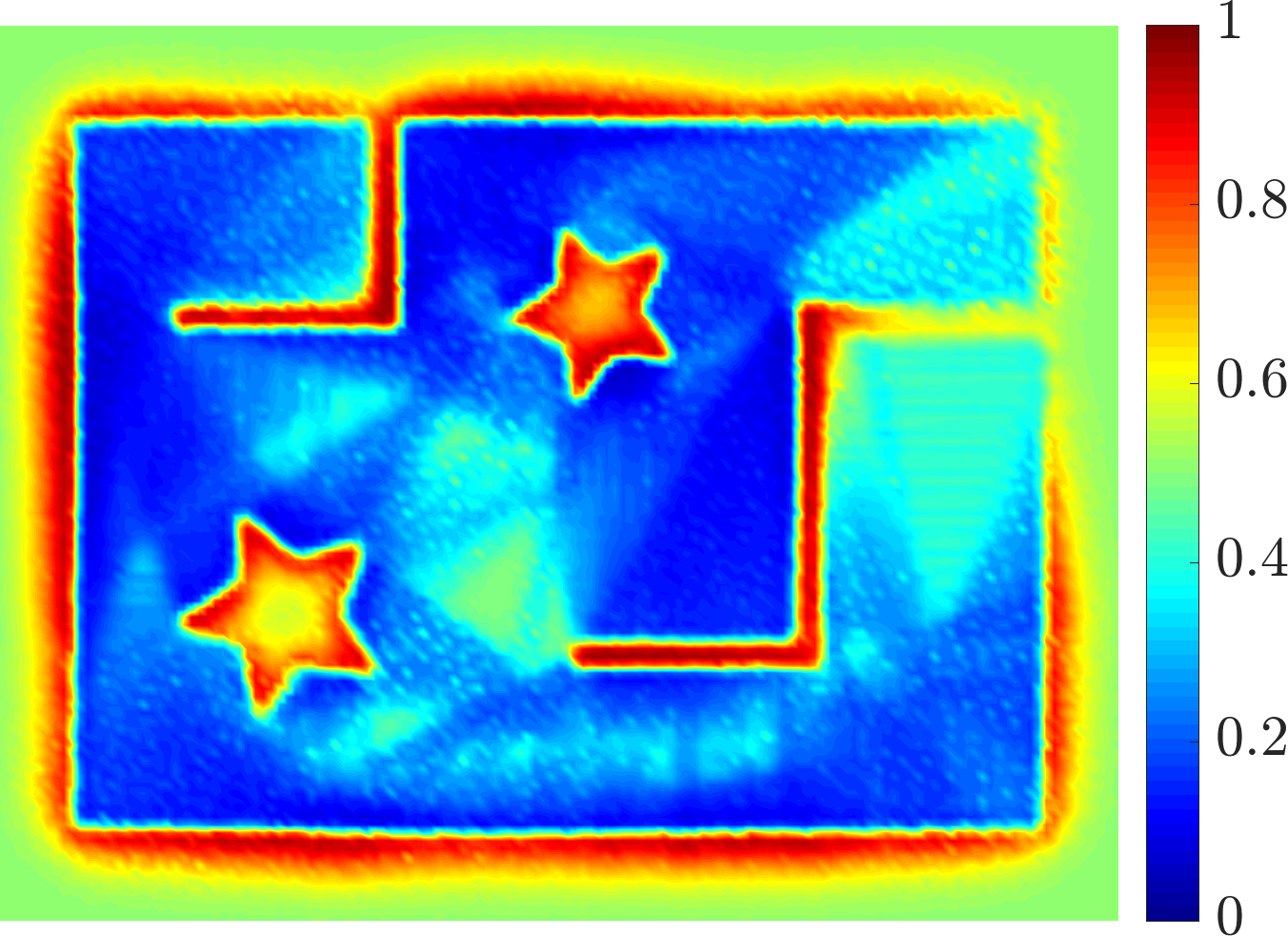}
    \label{fig:gpom_startest}
    }
  \subfloat[]{
    \includegraphics[width=.53\columnwidth]{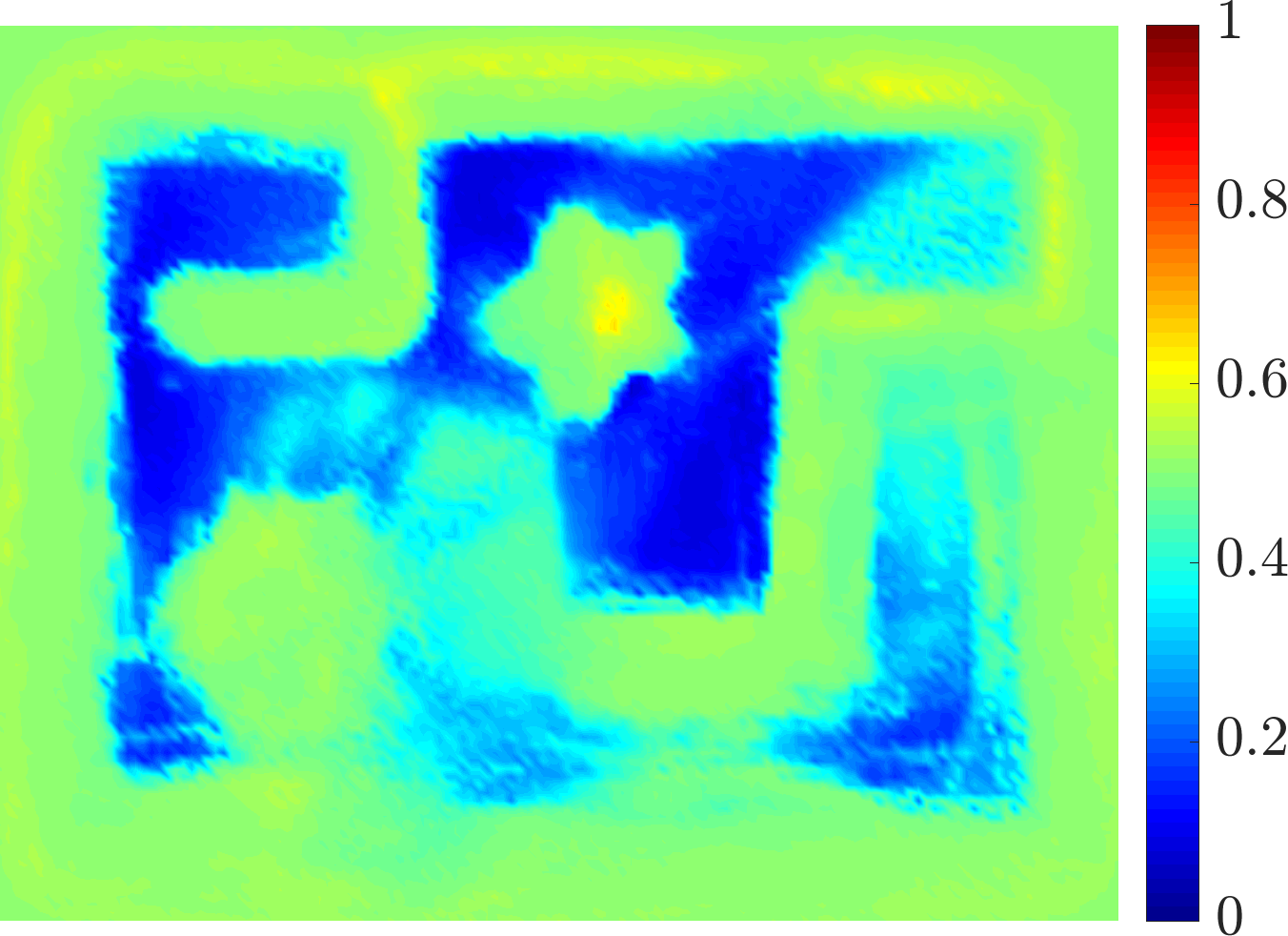}
    \label{fig:gpom_startest_esm}
    }
  \subfloat[]{
    \includegraphics[width=.53\columnwidth]{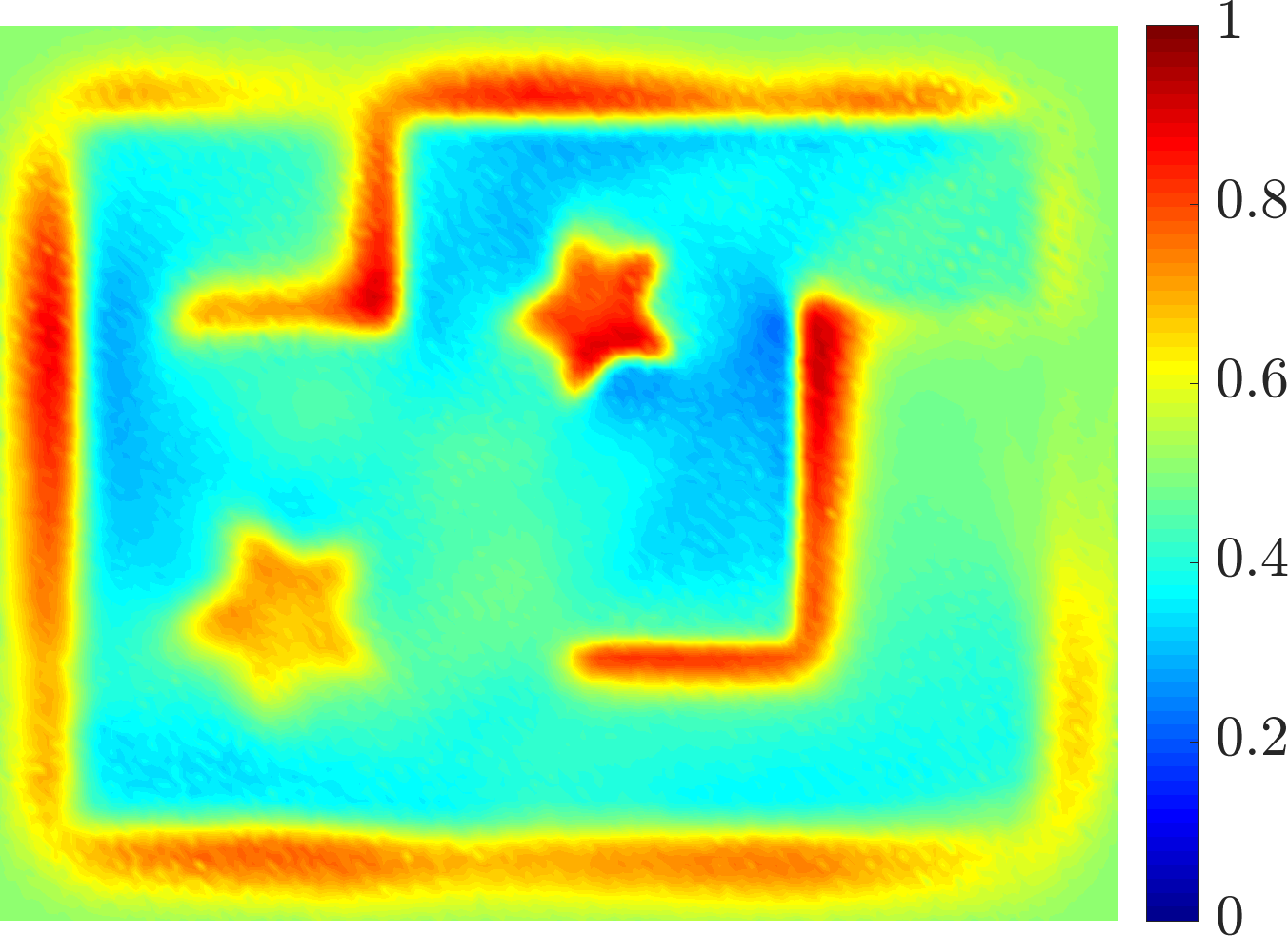}
    \label{fig:gpom_startest_ek}
    }\\
  \subfloat[]{
    \includegraphics[width=.53\columnwidth]{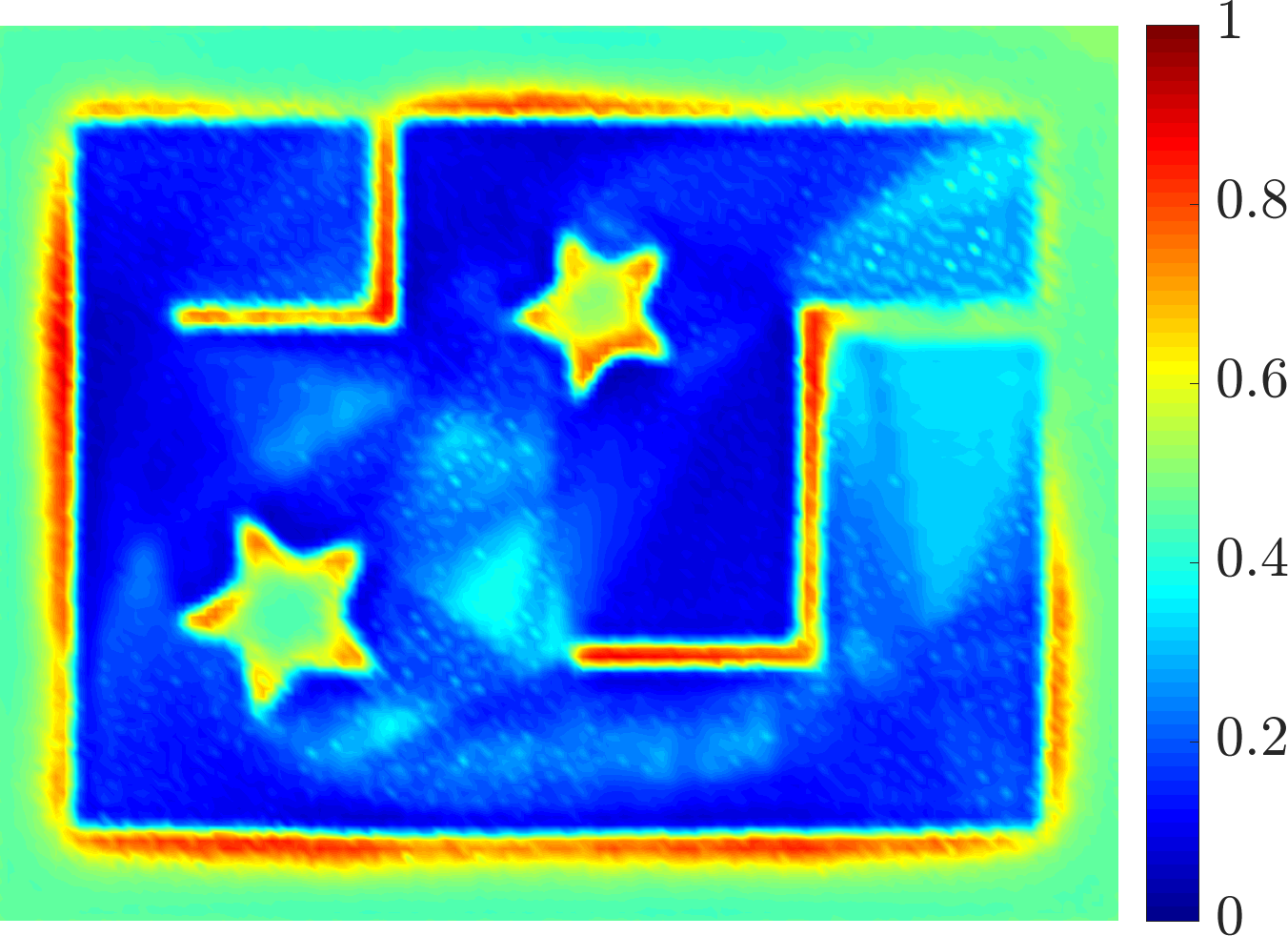}
    \label{fig:wgpom_startest}
    }
  \subfloat[]{
    \includegraphics[width=.53\columnwidth]{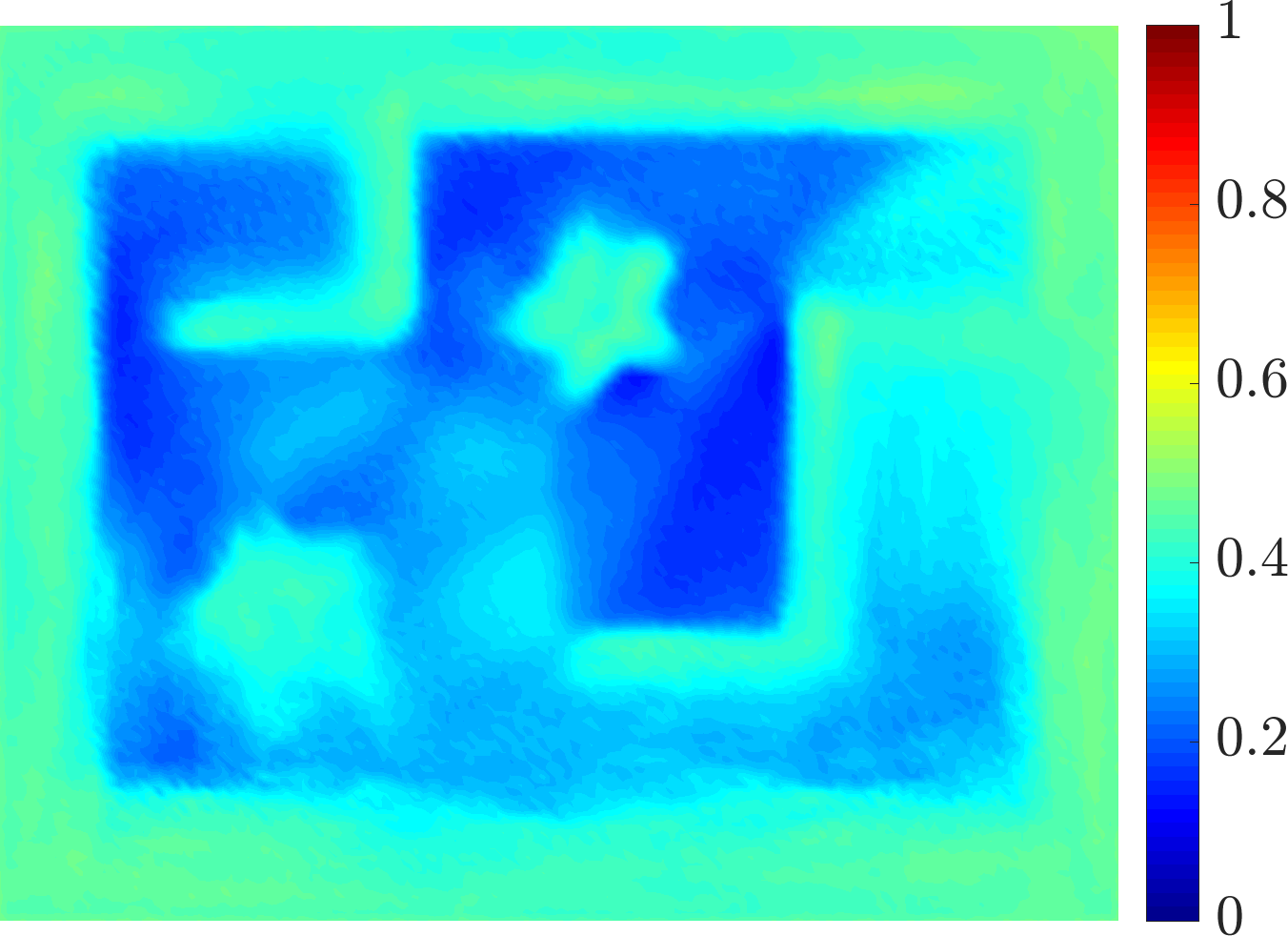}
    \label{fig:wgpom_startest_esm}
    }
  \subfloat[]{
    \includegraphics[width=.53\columnwidth]{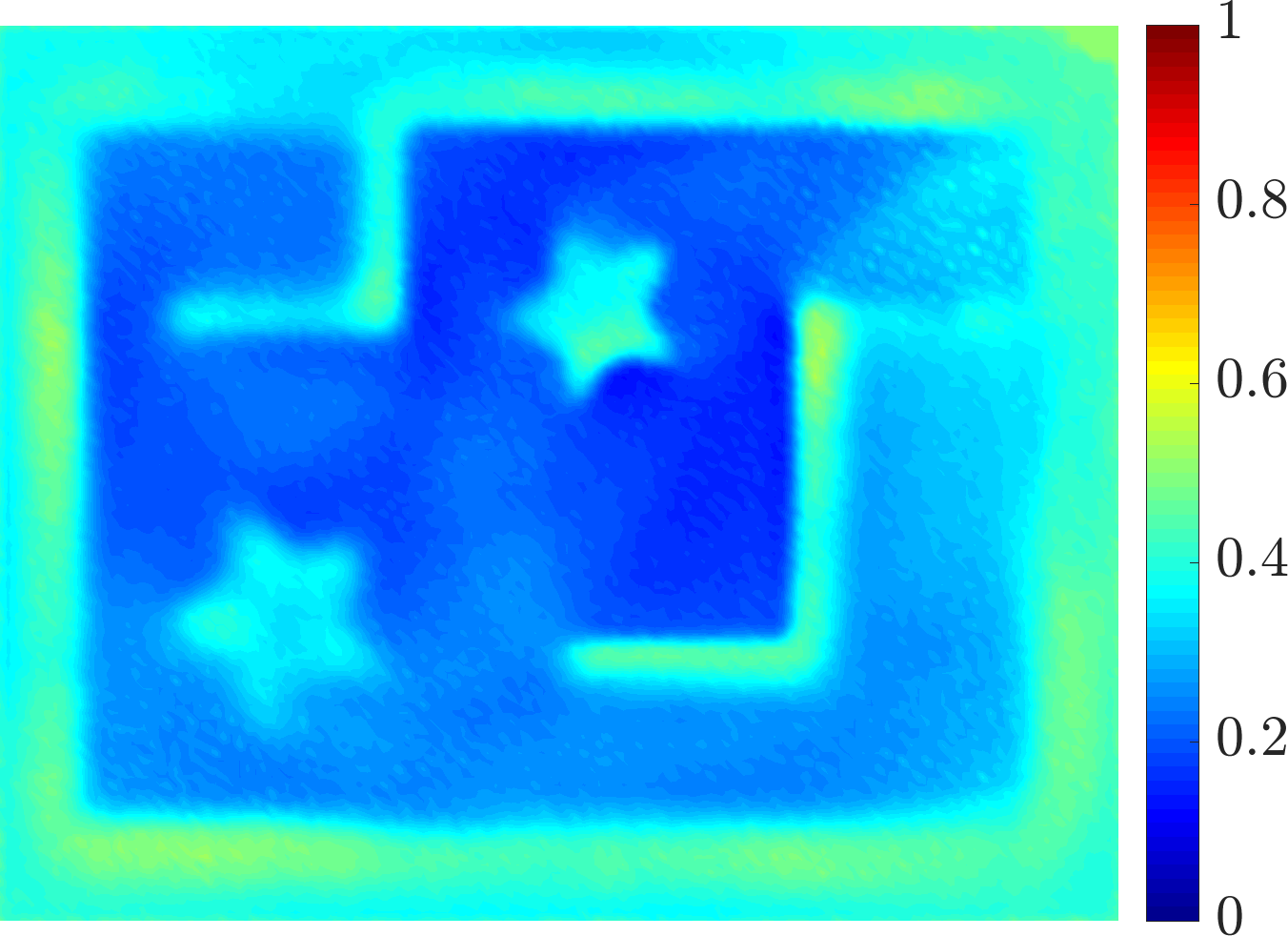}
    \label{fig:wgpom_startest_ek}
    }
  \caption{Illustrative examples of the occupancy map from the first experiment (for $\boldsymbol Q_3$). The top row corresponds to the GPOM, and the bottom row shows WGPOM results. The maps in (a) and (d) show GPOM and WGPOM results by ignoring the robot pose uncertainties. In (b) and (e) the robot pose uncertainty is incorporated using the expected sub-map method. In (c) and (f) the robot pose uncertainty is incorporated using the expected kernel method. The WGPOM-based maps can deal with input uncertainty better and provide maps with higher quality, shown in (e) and (f). All maps are computed using $0.5\m$ resolution.}
  \label{fig:startest}
\end{figure*}

\subsection{First Experiment: Motion Uncertainty Effect}
\label{subsec:wmapres}
The map of the environment, the robot trajectory, the observations collected at each pose using a simulated rangefinder sensor, together with the evolution of the robot pose uncertainty due to its motion noise, are illustrated in Figure~\ref{fig:synmap_setup}. The uncertainty ellipsoids show the worst-case covariance of the robot position, and there is no uncertainty reduction along the path by closing loops.

Details from the model selection, compared techniques, and conducted experiments are collected in Table~\ref{tab:starexpsetup}. We use Mat\'ern ($\nu = 5/2$) covariance function for GPOM as its performance has been shown fitting in earlier works~\cite{jadidi2016gaussian}. For WGPOM, we use SE covariance function with Automatic Relevance Determination (ARD)~\cite{neal1996bayesian}. Since the observations do not cover the entire map, we use $\tanh$ with $\ell=2$ as the warping function to improve the extrapolation ability of GPs (Figure~\ref{fig:toy_eg}). The experiment for each mapping technique using EK and ESM uncertainty propagation is repeated by increasing the robot motion noise covariance in five steps. In Figure~\ref{fig:synmap_qplot}, the map accuracy and runtime comparisons of all methods using AUC are shown. In the bottom plot, the runtime for the ESM is higher than the EK for both mapping techniques. From the top plot, we can see that applying WGPs improves the map quality regardless of the uncertainty propagation choice. The expected kernel is computed using Gauss-Hermite quadrature with $9$ sample points, and the expected sub-map computations are performed using Monte-Carlo approximations with $10$ samples. When increasing the number of samples from $9$ to $18$ we did not observe any improvement in the map accuracy.

Figure~\ref{fig:startest} illustrates the results of all combinations of the proposed techniques. GPOM and WGPOM by ignoring the robot pose uncertainty are shown in Figures~\ref{fig:gpom_startest} and \ref{fig:wgpom_startest}, respectively. WGPOM demonstrates a better discrimination performance between class labels in the absence of measurements. This effect can be seen in the middle of the smaller star in both maps. Note that further optimization of hyperparameters can lead to over-fitting instead of solving the discussed problem. 

Even though ESM and EK try to achieve the same goal, they demonstrate different behaviors. Generally speaking, integration over the covariance function (EK) has a smoothing effect that sometimes can be desirable. For example, in all the presented maps the central part of the map due to the lack of observation is partially complete. As a result of this smoothing effect of the EK, this part is correctly classified to be closer to the free space class. However, the probabilities are closer to $0.5$, and the robot cannot be completely confident about the status of the area. Alternatively, ESM leads to relatively more confident maps with smaller smoothing effects. The resultant maps are safer for navigation as the gap between occupied and unoccupied areas is classified as an unknown region. While this behavior can be an appealing property from a motion planning point of view, the occupied areas are faded; and we cannot see the structural shape of the environment accurately. Overall, by propagating more uncertainty into the map inference process, we expect less accuracy in the outcome and more realistic estimation of the belief. However, WGPs by modeling nonlinearity in the observation space improve the map quality.

\subsection{Experimental Results}
\label{subsec:wmapexpres}
The experimental results of occupancy mapping using the Intel dataset are shown in Table~\ref{tab:intelaucroc} and Figure~\ref{fig:Intel_warp}. In the absence of a complete groundtruth map, the groundtruth map for this dataset is generated using the estimated robot trajectory and rangefinder measurements. In this way, the groundtruth map has the same orientation which makes the comparison convenient. The covariance function used is an intrinsically sparse kernel~\cite{melkumyan2009sparse}. The logic behind this choice is that the structural shape of the environment is complex, and it is cluttered with random people and typical office furniture; therefore, using a covariance function that correlates map points over a long range is not suitable. We use the Sparse covariance function for both GPOM and WGPOM and their corresponding uncertainty propagation experiments.

\begin{table}
\footnotesize
\centering
\caption{The AUC and runtime for incremental GPOM and WGPOM using the Intel dataset and proposed uncertainty propagation methods. The runtimes are in minutes. Maps are computed using $0.2\m$ resolution and a $k$d-tree data structure for the storage and nearest neighbor queries. The Sparse covariance function is used for all techniques and Polynomials of degree seven as the warping function. Total and average number of training points, and the number of test points are $n_{t} = 137979$, $\bar{n}_{t} = 186$, and $n_{q}= 40401$, respectively.}
\begin{tabular}{lcccccc}
\toprule
Method	&\multicolumn{2}{c}{No uncertainty} & \multicolumn{2}{c}{EK} & \multicolumn{2}{c}{ESM} \\ \midrule
				& AUC		& Time 	& AUC		& Time 	& AUC		& Time  \\ \midrule
GPOM			& 0.8499	& 72	& 0.7323	& 532 	& 0.7026	& 657 	\\
WGPOM			& 0.8463	& 125	& 0.7887	& 577 	& 0.7436	& 766 \\ \bottomrule
\end{tabular}
\label{tab:intelaucroc}
\end{table}

Table~\ref{tab:intelaucroc} shows the runtime and accuracy comparison of the mapping techniques. In this scenario, where the pose uncertainties are ignored, GPOM performs marginally better than WGPOM. This result, shown in Figures~\ref{fig:intel_gpom} and \ref{fig:intel_wgpom}, can be understood from the fact that WGPOM has covered more partially observed areas in the map. While this is desirable, in the absence of a complete groundtruth map, it leads to a lower AUC. However, WGPOM maintains more accurate maps by incorporating the pose uncertainties which is the actual problem to be solved.

\begin{figure*}
  \centering  
  \subfloat[]{
    \includegraphics[width=0.58\columnwidth]{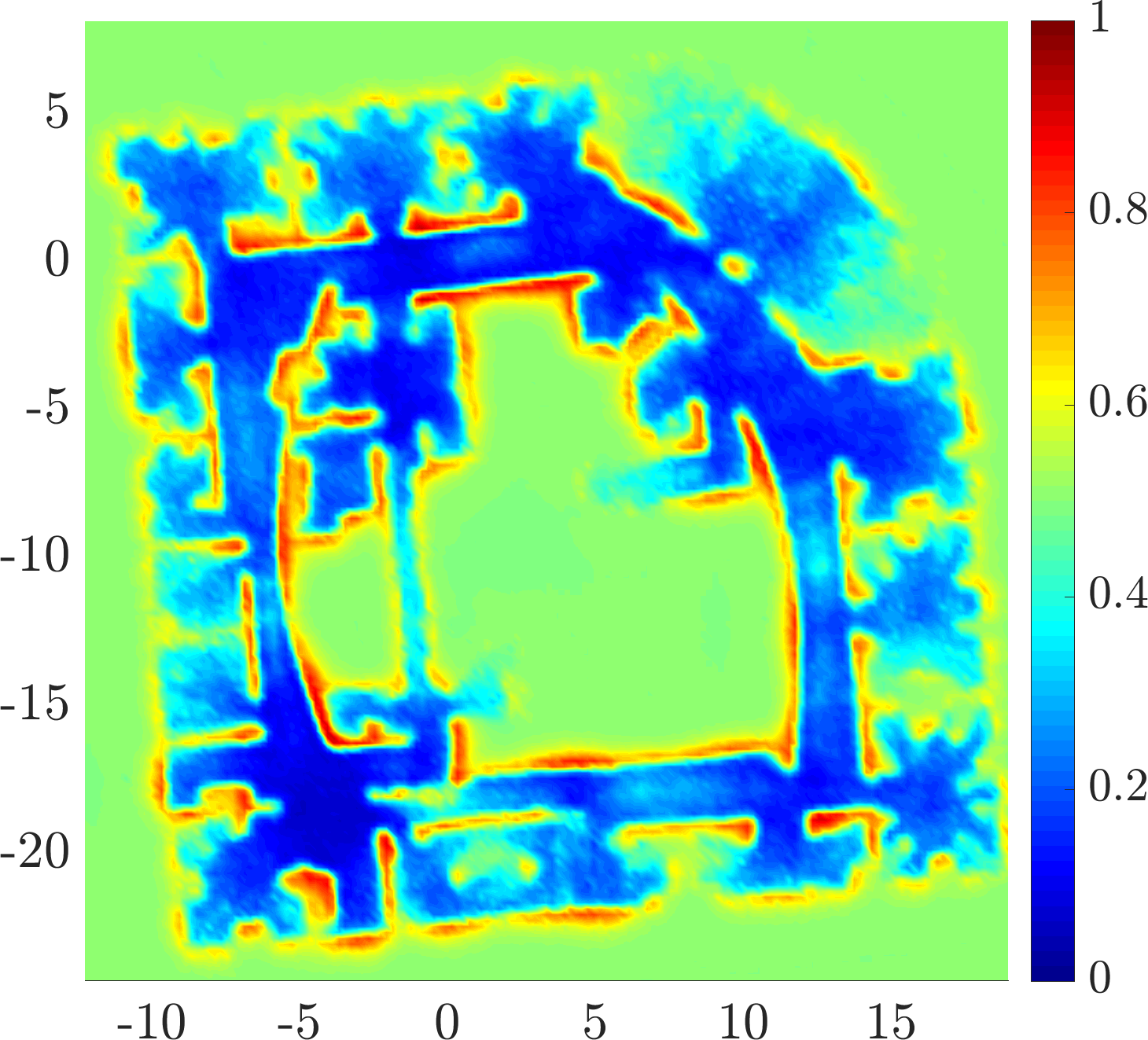}
    \label{fig:intel_gpom}
    }
  \subfloat[]{
    \includegraphics[width=0.58\columnwidth]{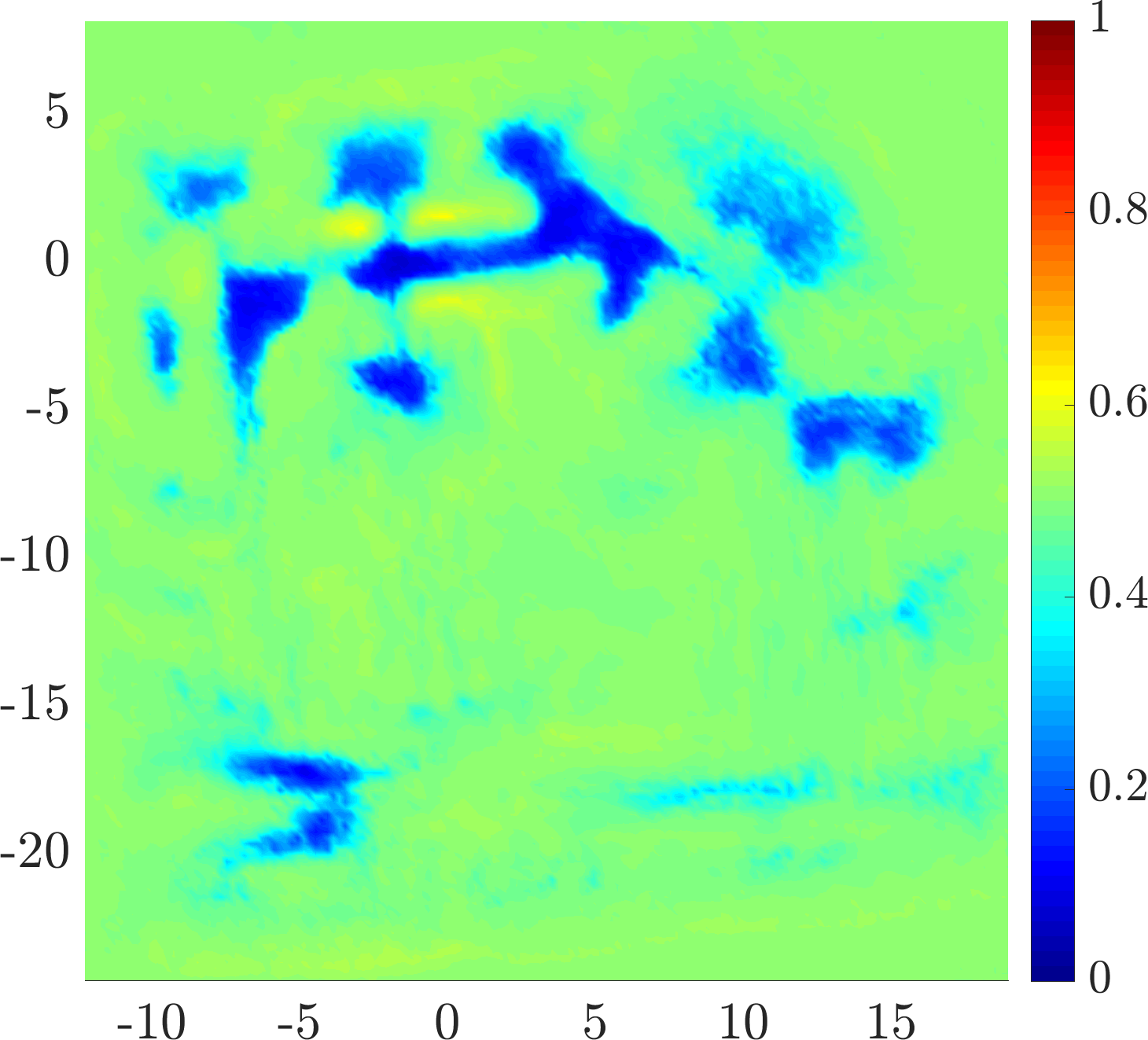}
    \label{fig:intel_esm}
    }
  \subfloat[]{
    \includegraphics[width=0.58\columnwidth]{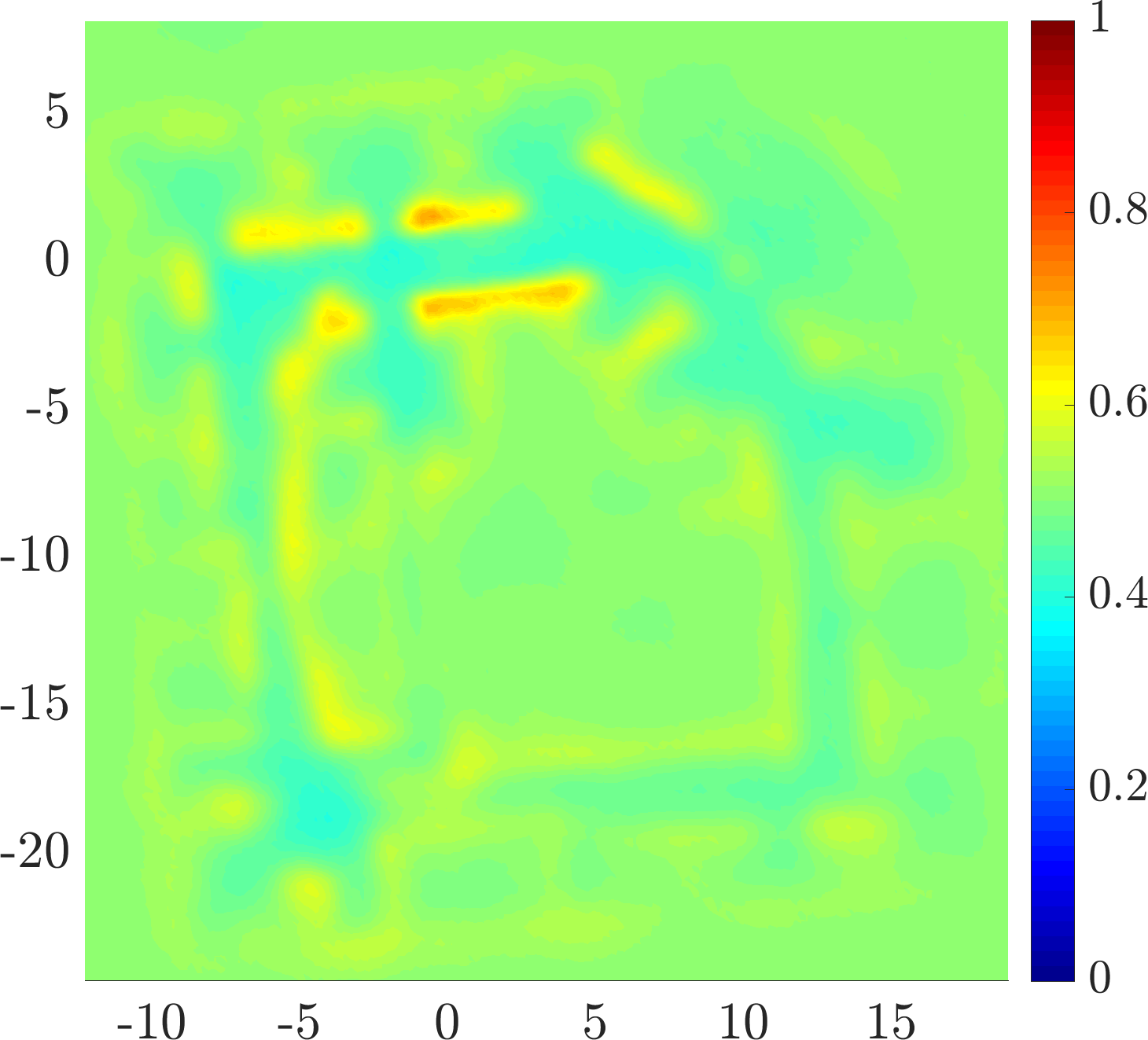}
    \label{fig:intel_ek}
    }\\
  \subfloat[]{
    \includegraphics[width=0.58\columnwidth]{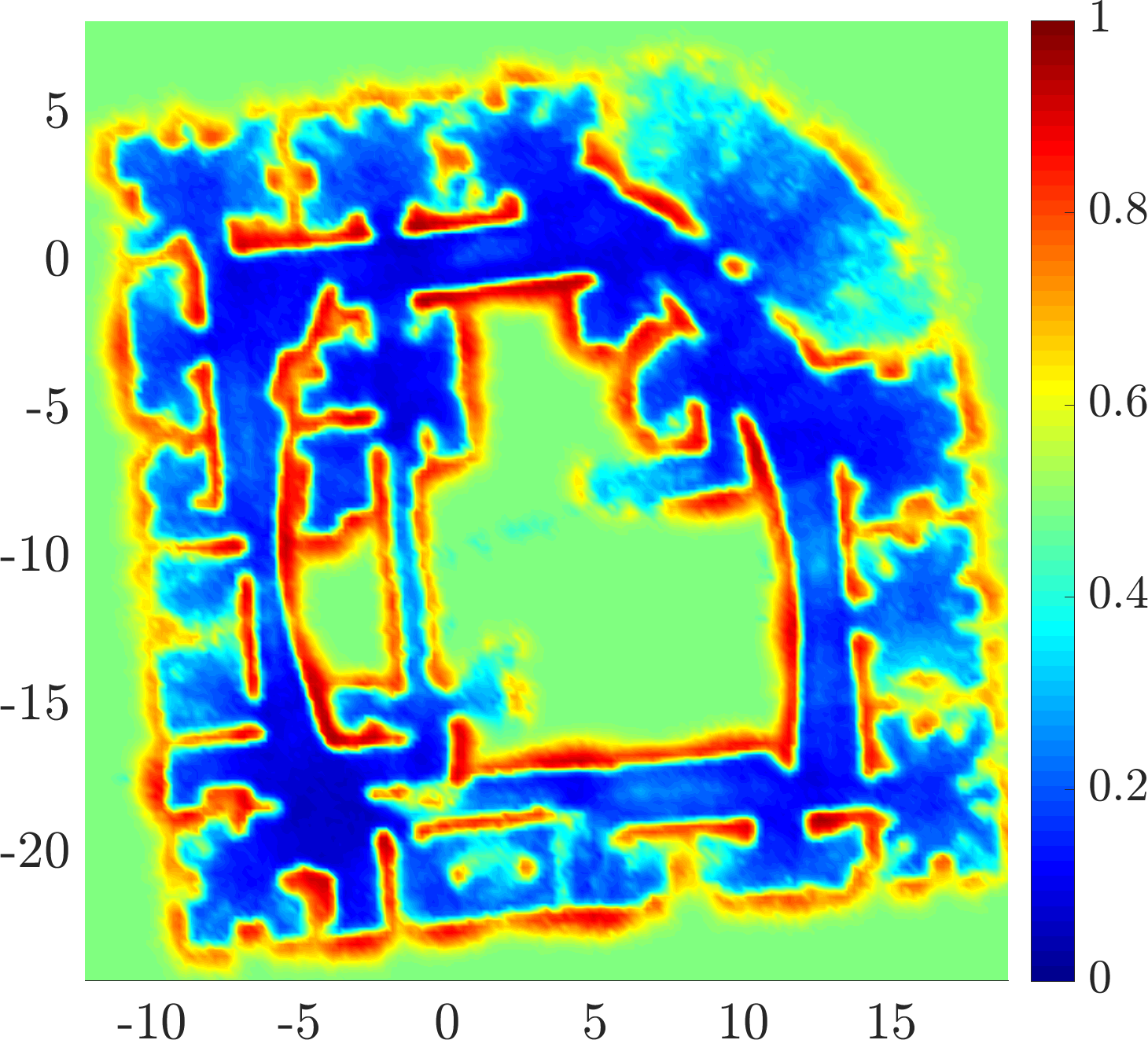}
    \label{fig:intel_wgpom}
    }
  \subfloat[]{
    \includegraphics[width=0.58\columnwidth]{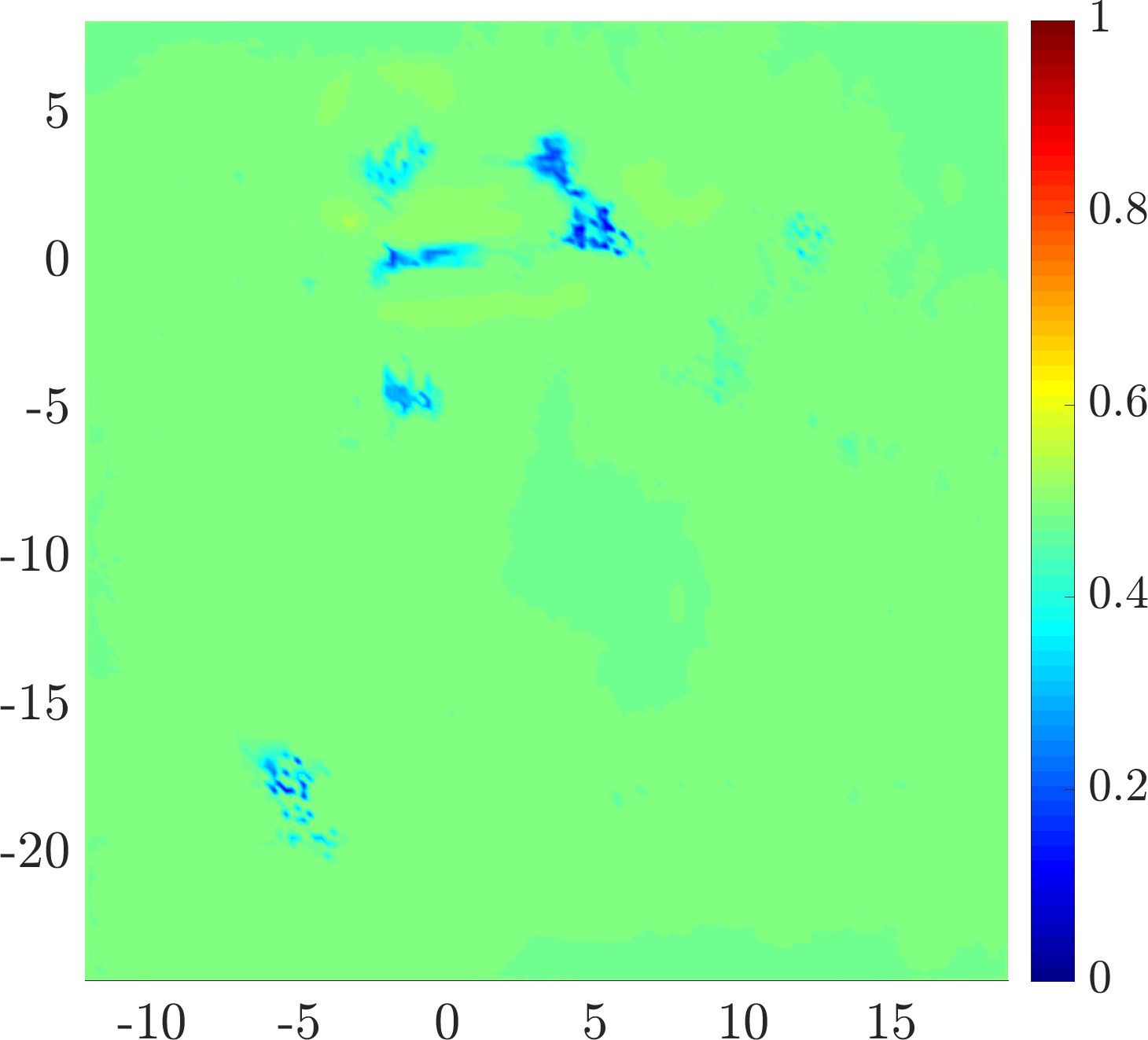}
    \label{fig:intel_wesm}
    }
  \subfloat[]{
    \includegraphics[width=0.58\columnwidth]{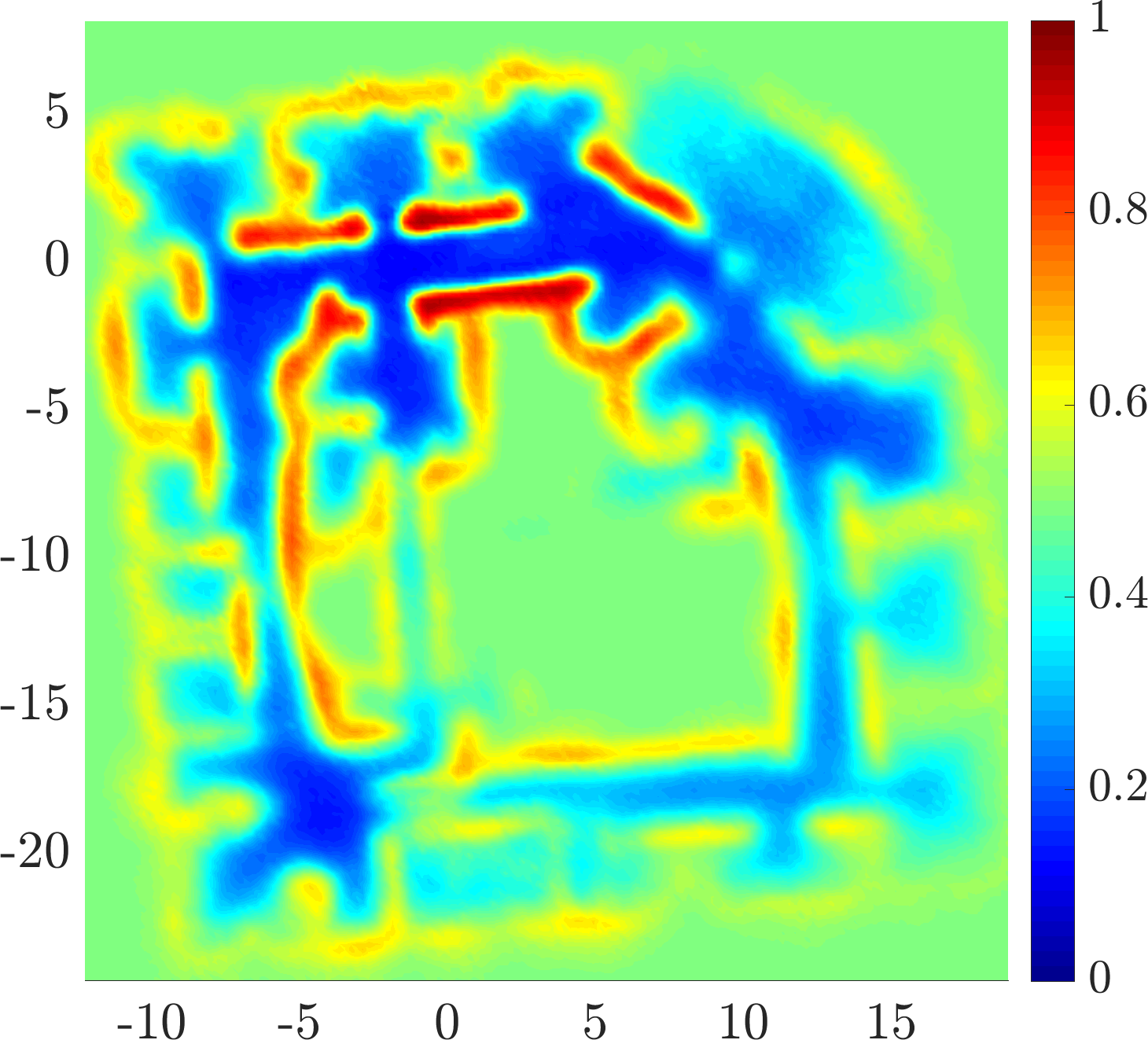}
    \label{fig:intel_wek}
    }
  \caption{Occupancy mapping results using the Intel dataset. The top row corresponds to the GPOM, and the bottom row shows WGPOM results. The maps built by ignoring pose uncertainties are shown in (a) and (d) using GPOM and WGPOM, respectively.  In (b) and (e) the robot pose uncertainty is incorporated using the expected sub-map method. In (c) and (f) the robot pose uncertainty is incorporated using the expected kernel method.}
  \label{fig:Intel_warp}
\end{figure*}

In this experiment, the uncertainty propagation using the ESM method leads to poor map qualities and the maps, shown in Figures~\ref{fig:intel_esm} and \ref{fig:intel_wesm}, are almost entirely faded due to the significant pose uncertainties resulting from the robot repeatedly traveling through the same areas. This fading effect could be partially mitigated by increasing the number of samples, but based on the runtimes in Table~\ref{tab:intelaucroc}, it is not justifiable. Nevertheless, WGPOM using the EK demonstrates a better performance in comparison to all methods that incorporated pose uncertainties; moreover, it produces a map that is also usable for navigation tasks, i.e. obstacle avoidance, and path planning.

\subsection{Discussion and Limitations}
\label{subsec:wmaplimit}
Uncertainty propagation through the developed methods can provide safer maps for robotic navigation. However, long-term uncertainty propagation leads to highly faded maps. If the uncertainty at each step is large and the robot cannot improve the localization confidence through loop-closures, this fading effect is more severe. Also, the uncertainty of the robot orientation has a large impact on the map quality. As we have seen in the presented results, in the expected sub-map approach, if the orientation uncertainty is significant, integration using a small number of samples leads to poor map accuracies. The expected kernel technique has an advantage from this perspective, as the unscented transform maps the measurement into the map space, and samples are related to the map spatial dimensions.

It is also demonstrated that the concept of accounting for possible nonlinearities in the observation space, here through the warping function, has desirable effects on the map quality. Our results showed that regardless of the uncertainty propagation technique, applying WGPs provide better maps than standard GPs. We reiterate that by ignoring the robot pose uncertainty, the map is a potentially incorrect representation of the environment.

Finding an appropriate warping function that is compatible with non-Gaussianity in the observation space can be time-consuming unless the model selection is performed in a more systematic way. In this work, we tried several functions and chose the best one. Moreover, since the exact inference is not possible, approximate methods may not always converge. Although the upper-bound time complexity of WGPOM is similar to that of GPOM, in practice for larger datasets the inference takes longer. Improving the computational efficiency of the proposed methods is an interesting direction to follow.

\subsection{Computational Complexity}
\label{subsec:complexity}
For both GPOM and WGPOM, the worst-case time complexity is cubic in the number of training data, $\mathcal{O}(\bar{n}_{t}^3)$. For ESM, the number of sub-map fusion into the global map scales linearly with the number of samples, $n_s$, and the sub-map fusion involves a nearest neighbor query for each test point resulting in $\mathcal{O}(\bar{n}_{t}^3 + n_s n_{q} \log n_{q})$. In the case of two-dimensional mapping, using Gauss-Hermite quadrature, the time complexity of EK computation is quadratic in the number of sample points and together with sub-map fusion leads to $\mathcal{O}(\bar{n}_{t}^3 + n_s^2 + n_{q} \log n_{q})$. In EK, applying the unscented transform to all training points involves a Cholesky factorization of the input covariance matrix which is ignored for the two-dimensional case.

\section{CONCLUSION}
\label{sec:conclusion}

In this paper, we studied incremental GP occupancy mapping extensions through warped GPs. Since occupancy maps have dense belief representations, the robot pose uncertainty is often ignored. We proposed two methods to incorporate robot pose uncertainty into the map inference, the expected kernel and the expected sub-map. While the expected kernel handles the input uncertainty within the GPs framework, the expected sub-map exploits the inherent property of stationary covariance functions for map inference in the local frame with deterministic inputs. The proposed methods can also be useful if the belief representation is not dense (as opposed to occupancy mapping). Furthermore, the WGPOM technique can deal with the nonlinear behavior of measurements through a nonlinear transformation which improves the ability of GPs to learn complex structural shapes more accurately, especially, under uncertain inputs.

Future work includes further examinations of the proposed methods in practical robotic exploration and obstacle avoidance scenarios, for example, using currently popular visual-inertial odometry systems.

% \addtolength{\textheight}{-12cm}   % This command serves to balance the column lengths
                                  % on the last page of the document manually. It shortens
                                  % the textheight of the last page by a suitable amount.
                                  % This command does not take effect until the next page
                                  % so it should come on the page before the last. Make
                                  % sure that you do not shorten the textheight too much.

% \section*{Acknowledgment}
% 
% The authors would like to thank three anonymous reviewers for their valuable feedback and Kasra Khosoussi for the helpful discussions.

\bibliographystyle{IEEEtran} 
\bibliography{refs}

\end{document}